\RequirePackage[l2tabu,orthodox]{nag}
\documentclass
[letterpaper,12pt,]
{article}

\usepackage{etex}
\usepackage{verbatim}
\usepackage{xspace,enumerate}
\usepackage[dvipsnames]{xcolor}
\usepackage[T1]{fontenc}
\usepackage[full]{textcomp}
\usepackage[american]{babel}
\usepackage{mathtools}
\usepackage{amsthm}
\usepackage[letterpaper, top=1in, bottom=1in, left=1in, right=1in]{geometry}
\usepackage{newpxtext} 
\usepackage{textcomp} 
\usepackage[varg,bigdelims]{newpxmath}
\usepackage[scr=rsfso]{mathalfa}
\usepackage{bm} 
\let\mathbb\varmathbb
\usepackage{microtype}
\usepackage[pagebackref,colorlinks=true,urlcolor=blue,linkcolor=blue,citecolor=OliveGreen]{hyperref}
\usepackage[capitalise,nameinlink]{cleveref}
\crefname{lemma}{Lemma}{Lemmas}
\crefname{fact}{Fact}{Facts}
\crefname{theorem}{Theorem}{Theorems}
\crefname{corollary}{Corollary}{Corollaries}
\crefname{claim}{Claim}{Claims}
\crefname{example}{Example}{Examples}
\crefname{algorithm}{Algorithm}{Algorithms}
\crefname{problem}{Problem}{Problems}
\crefname{definition}{Definition}{Definitions}
\crefname{exercise}{Exercise}{Exercises}
\usepackage{amsthm}

\newtheorem{theorem}{Theorem}[section]
\newtheorem*{theorem*}{Theorem}
\newtheorem{lemma}[theorem]{Lemma}
\newtheorem*{lemma*}{Lemma}
\newtheorem{fact}[theorem]{Fact}
\newtheorem*{fact*}{Fact}

\newtheorem*{proposition*}{Proposition}

\newtheorem*{corollary*}{Corollary}

\newtheorem*{hypothesis*}{Hypothesis}

\newtheorem*{conjecture*}{Conjecture}
\theoremstyle{definition}

\newtheorem*{definition*}{Definition}

\newtheorem*{construction*}{Construction}

\newtheorem*{example*}{Example}

\newtheorem*{question*}{Question}
\newtheorem{algorithm}[theorem]{Algorithm}
\newtheorem*{algorithm*}{Algorithm}

\newtheorem*{assumption*}{Assumption}

\newtheorem*{problem*}{Problem}

\newtheorem*{openquestion*}{Open Question}

\newtheorem*{model*}{Model}
\theoremstyle{remark}

\newtheorem*{claim*}{Claim}
\newtheorem{remark}[theorem]{Remark}
\newtheorem*{remark*}{Remark}

\newtheorem*{observation*}{Observation}
\usepackage{paralist}
\let\originalleft\left
\let\originalright\right
\renewcommand{\left}{\mathopen{}\mathclose\bgroup\originalleft}
\renewcommand{\right}{\aftergroup\egroup\originalright}
\usepackage{turnstile}
\usepackage{mdframed}
\usepackage{tikz}
\usepackage{caption}
\DeclareCaptionType{Algorithm}
\usepackage{newfloat}
\usepackage{xparse}
\usepackage{amsthm} 
\makeatletter
\let\latexparagraph\paragraph
\RenewDocumentCommand{\paragraph}{som}{%
	\IfBooleanTF{#1}
	{\latexparagraph*{#3}}
	{\IfNoValueTF{#2}
		{\latexparagraph{\maybe@addperiod{#3}}}
		{\latexparagraph[#2]{\maybe@addperiod{#3}}}%
	}%
}
\newcommand{\maybe@addperiod}[1]{%
	#1\@addpunct{.}%
}
\makeatother

\usepackage{todonotes}

\usepackage{boxedminipage}
\newcommand{\paren}[1]{(#1)}
\newcommand{\Paren}[1]{\left(#1\right)}

\newcommand{\brac}[1]{[#1]}
\newcommand{\Brac}[1]{\left[#1\right]}

\newcommand{\abs}[1]{\lvert#1\rvert}
\newcommand{\Abs}[1]{\left\lvert#1\right\rvert}

\newcommand{\BIGabs}[1]{\bigg\lvert#1\bigg\rvert}

\newcommand{\Card}[1]{\left\lvert#1\right\rvert}

\newcommand{\Set}[1]{\left\{#1\right\}}

\newcommand{\norm}[1]{\lVert#1\rVert}
\newcommand{\Norm}[1]{\left\lVert#1\right\rVert}
\newcommand{\bignorm}[1]{\big\lVert#1\big\rVert}

\newcommand{\snorm}[1]{\norm{#1}^2}
\newcommand{\Snorm}[1]{\Norm{#1}^2}

\newcommand{\normo}[1]{\norm{#1}_1}
\newcommand{\Normo}[1]{\Norm{#1}_1}

\newcommand{\normi}[1]{\norm{#1}_\infty}

\newcommand{\normn}[1]{\norm{#1}_\textnormal{nuc}}
\newcommand{\Normn}[1]{\Norm{#1}_\textnormal{nuc}}

\newcommand{\normm}[1]{\norm{#1}_\textnormal{max}}
\newcommand{\Normm}[1]{\Norm{#1}_\textnormal{max}}

\newcommand{\normreg}[1]{\norm{#1}_\textnormal{reg}}
\newcommand{\Normreg}[1]{\Norm{#1}_\textnormal{reg}}

\newcommand{\iprod}[1]{\langle#1\rangle}

\newcommand{\Esymb}{\mathbb{E}}
\newcommand{\Psymb}{\mathbb{P}}

\DeclareMathOperator*{\E}{\Esymb}

\DeclareMathOperator*{\ProbOp}{\Psymb}
\renewcommand{\Pr}{\ProbOp}

\newcommand{\suchthat}{\;\middle\vert\;}

\newcommand{\sge}{\succeq}

\newcommand\bdot\bullet

\DeclareMathOperator{\Tr}{Tr}

\DeclareMathOperator*{\argminlimitversion}{argmin}

\DeclareMathOperator{\supp}{supp}

\DeclareMathOperator{\sign}{sign}

\DeclareMathOperator{\rank}{rank}

\newcommand{\iid}{i.i.d.\xspace}

\newcommand{\Hoelder}{H\"{o}lder\xspace}
\newcommand{\Holder}{\Hoelder}

\newcommand{\N}{\mathbb N}
\newcommand{\R}{\mathbb R}

\newcommand{\cB}{\mathcal B}
\newcommand{\cC}{\mathcal C}

\newcommand{\cE}{\mathcal E}
\newcommand{\cF}{\mathcal F}

\newcommand{\cK}{\mathcal K}
\newcommand{\cL}{\mathcal L}
\newcommand{\cM}{\mathcal M}
\newcommand{\cN}{\mathcal N}

\newcommand{\cP}{\mathcal P}

\newcommand{\cS}{\mathcal S}
\newcommand{\cT}{\mathcal T}
\newcommand{\cU}{\mathcal U}

\newcommand{\cZ}{\mathcal Z}

\newcommand{\bbP}{\mathbb P}

\renewcommand{\leq}{\leqslant}
\renewcommand{\le}{\leqslant}
\renewcommand{\geq}{\geqslant}
\renewcommand{\ge}{\geqslant}
\let\epsilon=\varepsilon
\numberwithin{equation}{section}
\newcommand\MYcurrentlabel{xxx}
\newcommand{\MYstore}[2]{%
	\global\expandafter \def \csname MYMEMORY #1 \endcsname{#2}%
}
\newcommand{\MYload}[1]{%
	\csname MYMEMORY #1 \endcsname%
}
\newcommand{\MYnewlabel}[1]{%
	\renewcommand\MYcurrentlabel{#1}%
	\MYoldlabel{#1}%
}
\newcommand{\MYdummylabel}[1]{}
\newcommand{\torestate}[1]{%
	\let\MYoldlabel\label%
	\let\label\MYnewlabel%
	#1%
	\MYstore{\MYcurrentlabel}{#1}%
	\let\label\MYoldlabel%
}
\newcommand{\restatetheorem}[1]{%
	\let\MYoldlabel\label
	\let\label\MYdummylabel
	\begin{theorem*}[Restatement of \cref{#1}]
		\MYload{#1}
	\end{theorem*}
	\let\label\MYoldlabel
}
\newcommand{\restatelemma}[1]{%
	\let\MYoldlabel\label
	\let\label\MYdummylabel
	\begin{lemma*}[Restatement of \cref{#1}]
		\MYload{#1}
	\end{lemma*}
	\let\label\MYoldlabel
}
\newcommand{\restateprop}[1]{%
	\let\MYoldlabel\label
	\let\label\MYdummylabel
	\begin{proposition*}[Restatement of \cref{#1}]
		\MYload{#1}
	\end{proposition*}
	\let\label\MYoldlabel
}
\newcommand{\restatefact}[1]{%
	\let\MYoldlabel\label
	\let\label\MYdummylabel
	\begin{fact*}[Restatement of \cref{#1}]
		\MYload{#1}
	\end{fact*}
	\let\label\MYoldlabel
}
\newcommand{\restate}[1]{%
	\let\MYoldlabel\label
	\let\label\MYdummylabel
	\MYload{#1}
	\let\label\MYoldlabel
}

\newcommand{\sse}{\subseteq}

\newcommand{\eps}{\epsilon}

\allowdisplaybreaks
\newcommand*{\Id}{\mathrm{Id}}

\newcommand*{\normf}[1]{\norm{#1}_{\mathrm{F}}}
\newcommand*{\Normf}[1]{\Norm{#1}_{\mathrm{F}}}
\newcommand*{\bignormf}[1]{\bignorm{#1}_{\mathrm{F}}}

\newcommand*{\transpose}[1]{{#1}{}^{\mkern-1.5mu\mathsf{T}}}

\renewcommand{\ij}{{ij}}
\providecommand{\rspan}[1]{\text{rspan}(#1)}
\providecommand{\cspan}[1]{\text{cspan}(#1)}

\newcommand{\ind}[1]{\bm{1}_{\Brac{#1}}}

\title{Consistent Estimation for PCA and Sparse Regression with Oblivious Outliers\thanks{This project has received funding from the European Research Council (ERC) under the European Union's Horizon 2020 research and innovation programme (grant agreement No 815464).}}

\author{
	Tommaso d'Orsi\thanks{ETH Z\"urich}
	\and
	Chih-Hung Liu\footnotemark[2]
	\and
	Rajai Nasser\footnotemark[2]
	\and
	Gleb Novikov\footnotemark[2]
	\and
	David Steurer\footnotemark[2]
	\and
	Stefan Tiegel\footnotemark[2]
}

\date{}

\begin{document}
	
	\maketitle
	\begin{abstract}
		We develop machinery to design efficiently computable and \emph{consistent} estimators, achieving estimation error approaching zero as the number of observations grows, when facing an oblivious adversary that may corrupt responses in all but an $\alpha$ fraction of the samples.
As concrete examples, we investigate two problems: 
sparse regression and principal component analysis (PCA).
For sparse regression, we achieve consistency for optimal sample size $n\gtrsim (k\log d)/\alpha^2$ 
and optimal error rate $O(\sqrt{(k\log d)/(n\cdot \alpha^2)})$
where $n$ is the number of observations, $d$ is the number of dimensions and $k$ is the sparsity of the parameter vector, allowing the fraction of inliers to be inverse-polynomial in the number of samples.
Prior to this work, no estimator was known to be consistent when the fraction of inliers $\alpha$ is $o(1/\log \log n)$, even for (non-spherical) Gaussian design matrices.
Results holding under weak design assumptions and in the presence of such general noise have only been shown in dense setting (i.e., general linear regression) very recently by d'Orsi et al.~\cite{ICML-linear-regression}.
In the context of PCA, we attain optimal error guarantees under broad spikiness assumptions on the parameter matrix (usually used in matrix completion). 
Previous works could obtain non-trivial guarantees only under the assumptions that the measurement noise corresponding to the inliers is polynomially small in $n$ (e.g., Gaussian with variance $1/n^2$).

To devise our estimators, we equip the Huber loss with non-smooth regularizers such as the $\ell_1$ norm or the nuclear norm, and extend d'Orsi et al.'s approach~\cite{ICML-linear-regression} in a novel way to analyze the loss function.
Our machinery appears to be easily applicable to a wide range of estimation problems.
We complement these algorithmic results with statistical lower bounds showing that the fraction of inliers that our PCA estimator can deal with is optimal up to a constant factor.

	\end{abstract}
	
	\clearpage
	%
	\microtypesetup{protrusion=false}
	\tableofcontents{}
	\microtypesetup{protrusion=true}
	\clearpage

	
	
	\section{Introduction}
Estimating information from structured data is a central theme in statistics that by now has found applications in a wide array of disciplines.
On a high level, a typical assumption in an estimation problem is the existence of a --known \textit{a priori}-- family of probability distributions $\cP:=\Set{\bbP_\beta\suchthat \beta\in \Omega}$ over some space $\cZ$ that are each indexed by some parameter $\beta\in \Omega$.
We then observe a collection of $n$ independent observations $\bm Z=\Paren{\bm Z_1,\ldots, \bm Z_n}$ drawn from an unknown probability distribution $\bbP_{\beta^*}\in \cP$.
The goal is to (approximately) recover the \emph{hidden} parameter $\beta^*$.\footnote{In this paper, we assume $\Omega\subseteq \R^d$ and $\cZ \subseteq \R^D$ for some $d,D$, denote random variables in bold face, and hide absolute constant factors with the notation $O(\cdot),\Omega(\cdot)\,,\gtrsim\,,\lesssim$ and logarithmic factors with $\tilde{O}(\cdot),\tilde{\Omega}(\cdot)$.}

Oftentimes, real-world data may contain  skewed, imprecise or corrupted measurements.
Hence, a desirable property for an estimator is to be robust to significant, possibly malicious, noise perturbations on the given observations.
Indeed, in the last two decades, a large body of work has been developed on designing robust algorithms (e.g see \cite{Ben-TalGN09,BertsimasBC11,MaronnaMYS19}). 
However, proving strong guarantees often either demands strong assumptions on the noise model or requires that the fraction of perturbed observations is small.
More concretely, when we allow the noise to be chosen adaptively, i.e., chosen dependently on the observations and hidden parameters, a common theme is that \emph{consistent} estimators -- estimators whose error tends to zero as the number of observations grows -- can be attained only when the fraction of outliers is small.

In order to make vanishing error possible in the presence of large fractions of outliers, it is necessary to consider weaker adversary models that are \emph{oblivious} to the underlying structured data.
In recent years, a flurry of works have investigated oblivious noise models \cite{DBLP:journals/jacm/CandesLMW11, DBLP:conf/isit/ZhouLWCM10, tsakonas2014convergence, DBLP:conf/nips/Bhatia0KK17,  sun_zhou18, SuggalaBR019, pensia2021robust, ICML-linear-regression}. These results however are tailor-made to the specific models and problems.
To overcome this limitation, in this paper we aim to provide a simple blueprint to design provably robust estimators under minimalistic noise assumptions for a large class of estimation problems.
As a testbed for our blueprint, we investigate two well-studied problems:
\begin{itemize}
	\item[\emph{Principal component analysis (PCA)}:] Given a matrix $\bm Y:= L^*+\bm N$ where $L^*\in \R^{n\times n}$ is an unknown parameter matrix and $\bm N$ is an $n$-by-$n$ \emph{random} noise matrix, the goal is to find an estimator $\hat{\bm L}$ for $L^*$ which is as close as possible to $L^*$ in Frobenius norm.
	\item[\emph{Sparse regression}:] Given observations $(X_1,\bm y_1),\ldots,(X_n,\bm y_n)$ following the linear model $\bm y_i =\iprod{X_i,\beta^*}+\bm \eta_i$ where $X_i \in\R^d$, $\beta^*\in \R^d$ is the $k$-sparse parameter vector of interest  (by $k$-sparse we mean that it has at most $k$ nonzero entries) and $\bm \eta_1,\ldots\bm \eta_n$ is noise, the goal is to find an estimator $\hat{\bm \beta}$ for $\beta^*$ achieving small \emph{squared prediction error} 
	$\frac{1}{{n}}\norm{X(\hat{\bm \beta}-\beta^*)}^2$, where $X$ is the matrix whose rows are $X_1,\ldots,X_n$.\footnote{Our analysis also works for the \emph{parameter error} $\norm{\beta^* - \hat{\bm \beta}}$.}
\end{itemize}

\paragraph{Principal component analysis}
A natural way to describe  principal component analysis under oblivious perturbations  is that of assuming the noise matrix $\bm N$ to be an $n$-by-$n$ matrix  with  a uniformly random  set   of $\alpha\cdot n^2$ entries bounded by some small $\zeta\geq 0$ in absolute value.
In these settings, we may think of the $\alpha$ fraction of entries with small noise as the set of uncorrupted observations. Moreover, for $\zeta>0$, the fact that even for uncorrupted observations the noise is  non-zero allows us  to capture both gross sparse errors and small entry-wise noise in the measurements at the same time (for example if $\zeta=1$ then the model captures settings with additional standard Gaussian noise).
Remarkably, for $\zeta=0$,  Cand\`{e}s et al.'s seminal work~\cite{DBLP:journals/jacm/CandesLMW11} provided an algorithm that exactly recovers $L^*$  even for a vanishing fraction of inliers, 
under \textit{incoherence} conditions on the signal $L^*$. The result was slightly extended in \cite{DBLP:conf/isit/ZhouLWCM10} where the authors provided an algorithm recovering $L^*$ up to squared error $O(\zeta^2\cdot n^4)$ thus allowing polynomially small measurement noise $\zeta$, but still failing to capture settings where  standard Gaussian measurement noise is added to the sparse noise. Even for simple signal matrices $L^*$, prior to this work, it remained an open question whether consistent estimators could be designed in presence of both oblivious noise and more reasonable measurement noise (e.g., standard Gaussian).

\paragraph{Linear regression}
Similarly to the context of principal component analysis, a convenient model for oblivious adversarial corruptions is that of assuming the noise vector $\bm \eta$ to have a random set of $\alpha\cdot n$ entries bounded by $\zeta$ in absolute value (here all results of interest can be extended to any $\zeta>0$ just by scaling, so we consider only the case $\zeta=1$).  This model trivially captures the classical settings with Gaussian noise (a Gaussian vector with variance $\sigma^2$ will have an $\alpha = \Theta(\frac{1}{\sigma})$ fraction of inliers) and, again, allows us to think of the $\alpha$-fraction of entries with small noise as the set of uncorrupted observations.
Early works on consistent regression focused on the regime with Gaussian design $\bm X_1,\ldots, \bm X_n\sim N(0, \Sigma)$ and  deterministic noise $\eta$.\footnote{For Gaussian design $\bm X$ one may also consider a deterministic noise model.  This noise model is subsumed by the random noise model discussed here (if $\bm X$ is Gaussian). As shown in \cite{ICML-linear-regression}, roughly speaking the underlying reason is that the Gaussianity of $\bm X$ allows one to obtain several other desirable properties "for free".  For example, one could ensure that the noise vector is symmetric by randomly flipping the sign of each observation $(\bm y_i, \bm X_i)$, as the design matrix will still be Gaussian. See \cref{sec:results} for a more in-depth discussion.}
\cite{DBLP:conf/nips/Bhatia0KK17} presented an estimator achieving error $\tilde{O}(d/(\alpha^2\cdot n))$ for any $\alpha$ larger than a fixed constant.
\cite{SuggalaBR019} extended this result by achieving comparable error rates even for a vanishing fraction of inliers $\alpha\gtrsim 1/\log\log n$. Assuming $n\gtrsim d^2/\alpha^2$ \cite{tsakonas2014convergence} proved that the Huber-loss estimator \cite{huber1964} achieves optimal error rate $O(d/\alpha^2\cdot n)$ even for polynomially small fraction of inliers $\alpha \gtrsim \sqrt{d/n}$.
This line of work culminated in \cite{ICML-linear-regression} which extended the result of \cite{tsakonas2014convergence} achieving the same guarantees with optimal sample complexity $n\gtrsim d/\alpha^2$. For Gaussian design, similar result as \cite{ICML-linear-regression} can be extracted from independent work \cite{pesme_neurips2021}. 
Furthermore, the authors of \cite{ICML-linear-regression} extended these guarantees to deterministic design matrices satisfying only a spreadness condition  (trivially satisfied by sub-Gaussian design matrices). 

Huber loss was also analyzed in context of linear regression in \cite{sun_zhou18, paristech, Sasai2020robust, pensia2021robust}. These works studied models that are different from our model, and these results do not work in the case $\alpha = o(1)$. \cite{sun_zhou18} used assumptions on the moments of the noise, \cite{paristech} and \cite{Sasai2020robust} studied the model with non-oblivious adversary, and the model from \cite{pensia2021robust} allows corruptions in covariates.

In the context of sparse regression, less is known. 
When the fraction of observations is constant $\alpha \geq \Omega(1)$, \cite{oblivious_sparse_regression} provided a first consistent estimator. Later \cite{paristech} and \cite{Sasai2020robust} improved that result (assuming $\alpha \geq \Omega(1)$, but these results also work with non-oblivious adversary).
In the case $\alpha = o(1)$, the algorithm of \cite{SuggalaBR019} (for Gaussian designs) achieves nearly optimal error convergence $\tilde{O}\Paren{(k\log^2 d)/\alpha^2 n}$, but requires $\alpha \gtrsim 1/\log \log n$. 
More recently, \cite{ICML-linear-regression} presented an algorithm for standard Gaussian design $\bm X\sim N(0,1)^{n\times d}$, achieving the nearly optimal error convergence $\tilde{O}(k/(\alpha^2\cdot n))$ for nearly optimal  inliers fraction $\alpha \gtrsim \sqrt{(k\cdot \log^2 d)/n}$.
Both \cite{SuggalaBR019} and \cite{ICML-linear-regression} use an iterative process, and require a bound $\norm{\beta^*}\le d^{O(1)}$ (for larger $\norm{\beta^*}$ these algorithms also work, but the fraction of inliers or the error convergence is worse).
The algorithm from \cite{ICML-linear-regression} , however, heavily relies on the assumption that $\bm X\sim N(0,1)^{n\times d}$ and appears unlikely to be generalizable to more general families of matrices (including non-spherical Gaussian designs).

\paragraph{Our contribution.} 
We propose new machinery to design efficiently computable \emph{consistent estimators} achieving optimal error rates and sample complexity against oblivious outliers.
In particular, we extend the approach of \cite{ICML-linear-regression} to structured estimation problems, by finding a way to exploit adequately the structure therein.
While consistent estimators have already been designed under more benign noise assumptions (e.g. the LASSO estimator for sparse linear regression under Gaussian noise), it was previously unclear how to exploit this structure in the setting of oblivious noise. 
One key consequence of our work is hence to demonstrate what minimal assumptions on the noise are sufficient to make effective recovery (in the sense above) possible.
Concretely, we show
\begin{itemize}
	\item[\textit{Oblivious PCA}:] Under mild assumptions on the noise matrix $\bm N$ and common assumption on the parameter matrix  $L^*$ --traditionally applied in the context of matrix completion \cite{DBLP:journals/jmlr/NegahbanW12}--  we provide an algorithm that achieves optimal error guarantees. 
	
	\item[\textit{Sparse regression}:] Under mild assumptions on the design matrix and the noise vector --similar to the ones used in   \cite{ICML-linear-regression} for dense parameter vectors $\beta^*$ --  we provide an algorithm that achieves optimal error guarantees and sample complexity.
\end{itemize}

For both problems, our analysis improves over the state-of-the-art and  recovers the classical optimal guarantees, not only for Gaussian noise, but also under much less restrictive noise assumptions. At a high-level, we achieve the above results by equipping the Huber loss estimator with appropriate regularizers.
Our techniques closely follow standard analyses for $M$-estimators, but crucially depart from them when dealing with the observations with large perturbations.
Furthermore, our analysis appears to be mechanical and thus easily applicable to many different estimation problems. 
	
\section{Results}\label{sec:results}

Our estimators are based on regularized versions of the \emph{Huber loss}.
The regularizer we choose depends on the underlying structure of the estimation problem:
We use $\ell_1$ regularization to enforce sparsity in linear regression and nuclear norm regularization to enforce a low-rank structure in the context of PCA.
More formally, the \textit{Huber penalty} is defined as the function $f_h:\R\rightarrow\R_{\geq 0}$ such that

\begin{align}\label{eq:huber-penalty}
f_h(t):=\begin{cases}
\frac{1}{2}t^2&\text{for }\abs{t}\leq h\,,\\
h(\abs{t}-\frac{h}{2})& \text{otherwise}.
\end{cases}
\end{align}

where $h>0$ is a penalty parameter. 
For $X\in \R^{D}$,  the \textit{Huber loss} is defined as the function $F_h(X):=\sum_{i \in [D]}f_h(X_i)$.
We will define the regularized versions in the following sections.
For a matrix $A$, we use $\norm{A}$, $\normn{A}$, $\normf{A}$, $\normm{A}$ to denote its spectral, nuclear, Frobenius, maximum\footnote{For $n\times m$ matrix $A$, $\normm{A} = \max_{i\in[n],j\in [m]} \abs{A_{ij}}$.} 
norms, respectively. For a vector $v$, we use $\norm{v}$ and $\normo{v}$ to denote its $\ell_2$ and $\ell_1$ norms.

\subsection{Oblivious principal component analysis}\label{sub:result-pca}
For oblivious PCA, we provide guarantees for the following estimator ($\zeta$ will be defined shortly): 
\begin{equation}\label{eq:pca-estimator}
\hat{\bm L} \coloneqq \qquad    \argminlimitversion_{\mathclap{\substack{L\in\R^{n\times n},\; \Normm{L}\leq \rho/n}}} \qquad  \Paren{F_h(\bm Y- L)+ 100\sqrt{n}\Paren{\zeta+\rho/n} \Normn{L}}.
\end{equation}

\begin{theorem}\label{thm:oblivious-pca}
	Let $L^*\in \R^{n\times n}$ be an unknown deterministic matrix and let $\bm N$ be an $n$-by-$n$ random matrix  with independent, symmetrically distributed (about zero) entries and $\alpha:=\min_{i,j \in [n]}\bbP \Set{\Abs{\bm N_\ij}\leq \zeta}$ for some $\zeta \ge 0$. Suppose that $\rank(L^*)=r$ and $\Normm{L^*}\leq \rho/n$. 
	 
	Then, with probability at least $1-2^{-n}$ over $\bm N$, given $\bm Y = L^*+\bm N$, $\zeta$ and $\rho$, the estimator \cref{eq:pca-estimator} with Huber parameter $h = \zeta + \rho/n$	satisfies 
	\[
	\displaystyle\Normf{\hat{\bm L}-L^*}\leq O\Paren{{\frac{\sqrt{rn}}{\alpha}}}\cdot(\zeta+\rho/n)\,.
	\]
\end{theorem}

We first compare the guarantees of \cref{thm:oblivious-pca} with the previous results on robust PCA~\cite{DBLP:journals/jacm/CandesLMW11,DBLP:conf/isit/ZhouLWCM10}.\footnote{We remark that in \cite{DBLP:journals/jacm/CandesLMW11} the authors showed that they can consider non-symmetric noise when the fraction of inliers is large $\alpha \ge 1/2$. However for smaller fraction of inliers their analysis requires the entries of the noise to be symmetric and independent, so for $\alpha <1/2$, their assumptions are captured by  \cref{thm:oblivious-pca}.}
The first difference is that they require $L^*$ to satisfy certain incoherent conditions.\footnote{A rank-$r$ $n\times n$ dimensional matrix $M$ is \emph{$\mu$-incoherent} if its \emph{singular vector decomposition} $M:=U\Sigma V^\top$ satisfies
	$\max_{i\in[n]}\norm{U^\top e_i}^2\leq \frac{\mu r}{n}$, $\max_{i\in[n]}\norm{V^\top e_i}^2\leq \frac{\mu r}{n}$ and $\normi{UV^\top}\leq \frac{\sqrt{\mu r}}{n}$.}
Concretely, they provide theoretical guarantees for $r\le O\big(\mu^{-1}n(\log n)^{-2}\big)$, where $\mu$ is the incoherence parameter.
In certain regimes, such a constraint strongly binds with the eigenvectors of $L^*$, restricting the set of admissible signal matrices. Using the different assumption $\normm{L^*}\leq \rho/n$ (commonly used for matrix completion, see \cite{DBLP:journals/jmlr/NegahbanW12}), we can obtain nontrivial guarantees (i.e. $\bignormf{\hat{\bm L}-L^*}/ \Normf{L^*}\to 0$ as $n\to \infty$) even when the $\mu$-incoherence conditions are not satisfied for any $\mu \le n/\log^2 n$, and hence the results \cite{DBLP:journals/jacm/CandesLMW11,DBLP:conf/isit/ZhouLWCM10} cannot be applied.
We remark that, without assuming incoherence,  the dependence of the error on the maximal entry of $L^*$ is inherent (see \cref{remark:incoherence}). 

The second difference is that \cref{thm:oblivious-pca} provides a significantly better dependence on the magnitude $\zeta$ of the entry-wise measurement error.
Specifically, in the settings of \cref{thm:oblivious-pca}, \cite{DBLP:conf/isit/ZhouLWCM10} showed that  if $L^*$ satisfies the {incoherence} conditions, the error of their estimator is $O\Paren{n^2\zeta}$. 
If the entries of $\bm N$ are standard Gaussian with probability $\alpha$ (and hence $\zeta \le O(1)$), and the entries of $L^*$ are bounded by $O(1)$, then the error of our estimator is $O(\sqrt{rn}/\alpha)$, which is considerably better than $O(n^2)$ as  in \cite{DBLP:conf/isit/ZhouLWCM10}. On the other hand, their error does not depend on the magnitude $\rho/n$ of the signal entries, so in the extreme regimes when the singular vectors of $L^*$ satisfy the incoherence conditions but $L^*$ has very large singular values (so that the magnitude of the entries of $L$ is  significantly larger than $n$), their analysis provides better guarantees than \cref{thm:oblivious-pca}.

As another observation to  understand  \cref{thm:oblivious-pca}, notice that our robust PCA model also captures the classical matrix completion settings.
In fact, any instance of matrix completion can be easily transformed into an instance of our PCA model: 
for the entries $(i,j)$ that we do not observe, we can set $N_{i,j}$ to some arbitrarily large value $\pm C(\rho, n) \gg \rho / n$, making the signal-to-noise ratio of the entry arbitrarily small. The observed (i.e. uncorrupted) $\Theta(\alpha\cdot n^2)$ entries may additionally be perturbed by Gaussian noise with variance $\Theta(\zeta^2)$.
The error guarantees of the estimator in  \cref{thm:oblivious-pca} is $O\Paren{\Paren{\rho/n + \zeta} \sqrt{r n} / \alpha}$. Thus, the dependency on  the parameters $\rho, n, \zeta$, and $r$ is the \textit{same} as in matrix completion and the error is within a factor of $\Theta(\sqrt{1/\alpha})$ from the optimum for matrix completion. However, this worse dependency on $\alpha$ is intrinsic to the more general model considered and it turns out to be optimal (see  \cref{thm:IT_lower_bound}). 
On a high level, the additional factor of $\Theta(\sqrt{1/\alpha})$ comes from the fact that in our PCA model we do not know which entries are corrupted. The main consequence of this phenomenon is that a condition of the form $\alpha \gtrsim \sqrt{r/n}$ appears \textit{inherent} to achieve consistency.
To get some intuition on why this condition is necessary, consider the Wigner model where we are given a matrix $\bm Y = xx^\top + \sigma\bm W$ for a flat vector $x\in \Set{\pm1}^n$ and a standard Gaussian matrix $\bm W$. 
Note, that the entries of $\bm W$ fit our noise model for $\zeta = 1\,, \rho/n=1\,, r=1$ and $\alpha= \Theta(1/\sigma)$. 
The spectral norm of $\sigma \bm W$ concentrates around $2\sigma\sqrt{n}$ and thus it is information-theoretically impossible to approximately recover the vector $x$ for $\sigma=1/\alpha=\omega(\sqrt{n})$ (see \cite{DBLP:journals/corr/PerryWBM16}).

\subsection{Sparse regression}\label{sub:result-sparse_regression}

Our regression model considers a \emph{fixed} design matrix $X \in \R^{n \times d}$ and observations $\bm y:= X \beta^* + \bm \eta \in\R^{n}$ where $\beta^*$ is an \emph{unknown} $k$-sparse parameter vector and $\bm \eta$ is \emph{random} noise with $\Psymb (\abs{\eta_i} \leq 1) \geq \alpha$ for all $i \in [n]$.
Earlier works \cite{DBLP:conf/nips/Bhatia0KK17}, \cite{SuggalaBR019} focused on the setting that the design matrix consists of \iid rows with Gaussian distribution $\cN(0, \Sigma)$ and the noise is $\eta = \zeta + \bm w$ where $\zeta$ is deterministic $(\alpha \cdot n)$-sparse vector and $\bm w$ is subgaussian.
As in   \cite{ICML-linear-regression}, our results for a fixed design and random noise can, in fact, extend to yield the same guarantees for this early setting (see \cref{thm:oblivious-sparse-regression-gaussian-design}).
Hence, a key advantage of our results is that the design $X$ does not have to consist of Gaussian entries.
Remarkably, we can handle arbitrary deterministic designs as long as they satisfy some mild conditions. Concretely, we make the following three assumptions, the first two of which are standard in the literature of sparse regression (e.g., see \cite{wainwright_2019}, section 7.3):

\begin{enumerate}
	\item 	For every column $X^i$ of $X$, $\Norm{X^i}\le \sqrt{n}$.
	\item\emph{Restricted eigenvalue property (RE-property)}:
	For every vector $u \in \R^d$ such that\footnote{For a vector $v \in \R^d$ and a set $S \sse [d]$, 
		we denote by $v_S$  the restriction of $v$ to the coordinates in $S$.} 
	$\Normo{u_{\supp(\beta^*)}} \geq 0.1\cdot \Normo{u}$, we have
	$\frac{1}{n}\Norm{Xu}^2\geq \lambda \cdot \Norm{u}^2$ for some parameter $\lambda >0$.
	\item\emph{Well-spreadness property}: For some (large enough) $m\in[n]$ and for every vector $u \in \R^d$ such that 
	$\Normo{u_{\supp(\beta^*)}} \geq 0.1\cdot \Normo{u}$ and for every subset $S\subseteq [n]$ with $\Card{S}\geq n - m$, it holds that $\Norm{(Xu)_S}\ge \frac{1}{2} \Norm{Xu}$.
\end{enumerate}

\definecolor{goodred}{rgb}{0.8, 0.0, 0.0}
Denote 
$\displaystyle\bm F_2(\beta):= \sum_{i=1}^n f_2\Paren{\bm y_i- \iprod{X_i, \beta}}\,$, 
where $X_i$ are the rows of $X$, and $f_2$ is as in \cref{eq:huber-penalty}. We devise our estimator for sparse regression and state its statistical guarantees below:

\begin{align}\label{eq:results-estimator-regression}
\hat{\bm  \beta} \coloneqq 
\arg \min_{\beta\in \R^d} \Paren{\bm F_2(\beta)+ 100\sqrt{n\log d}\cdot \Norm{\beta}_1}\,.
\end{align}

\begin{theorem}\label{thm:oblivious-sparse-regression}
	Let $\beta^* \in \R^d$ be an unknown $k$-sparse vector and let $X\in \R^{n\times d}$ be a deterministic matrix such that for each column $X^i$ of $X$, $\norm{X^i}\le \sqrt{ n}$, satisfying the RE-property with $\lambda>0$ and well-spreadness property with $m \gtrsim \frac{ k\log d}{\lambda\cdot \alpha^2}$ (recall that $n\ge m$).
	
	Further, let $\bm \eta$ be an $n$-dimensional random vector with independent, symmetrically distributed (about zero) entries and $\alpha=\min_{i\in[n]}\bbP\Set{\Abs{\bm \eta_i}\leq 1}$.
	
	Then with probability at least $1-d^{-10}$ over $\bm \eta$, given $X$ and $\bm y=X\beta^*+\bm \eta$, the estimator~\cref{eq:results-estimator-regression} satisfies 
	\[
	\displaystyle\frac{1}{n} \Norm{X\Paren{ \hat{\bm  \beta}- \beta^*}}^2 \leq O\Paren{\frac{1}{\lambda} \cdot \frac{ k\log d}{ \alpha^2 \cdot n}} 
	\mbox{\qquad and\qquad} 
	\displaystyle\Norm{{ \hat{\bm  \beta}- \beta^*}}^2 \leq O\Paren{\frac{1}{\lambda^2} \cdot \frac{ k\log d}{ \alpha^2 \cdot n}}\,.
	\]
\end{theorem}

There are important considerations when interpreting this theorem.
The first is the special case $\eta\sim N(0,\sigma^2)^{n}$,
which satisfies our model for $\alpha = \Theta(1/\sigma)$.
For this case, it is well known (e.g., see \cite{wainwright_2019}, section 7.3) that under the same RE assumption, the LASSO estimator achieves a prediction error rate of $O(\frac{\sigma^2 }{\lambda} \cdot \frac{k \log d}{ n}) = O(\frac{  k \log d}{\lambda\cdot\alpha^2\cdot  n})$, matching our result.
Moreover, this error rate  is essentially optimal. Under a standard assumption in complexity theory ($\textbf{NP}\not\subset \textbf{P}/\textbf{poly}$), the RE assumption is necessary when considering polynomial-time estimators~\cite{sparse_regression_lower_bound_prediction_error}.
Further, this also shows that the dependence on the RE constant seems unavoidable.
Under mild conditions on the design matrix, trivially satisfied if the rows are \iid Gaussian with covariance $\Sigma$ whose condition number is constant, our guarantees are optimal up to constant factors for \emph{all} estimators if $k\le d^{1-\Omega(1)}$ (e.g. $k\le d^{0.99}$), see \cite{sparse_regression_lower_bound_parameter_error}.
This optimality also shows that our bound on the number of samples is best possible since otherwise we would not be able to achieve vanishing error. The (non-sparse version of) well-spreadness property was first used  in the context of regression in  \cite{ICML-linear-regression}. In the same work the authors also showed that, under oblivious noise assumptions, some weak form of spreadness property is indeed necessary. 

The second consideration is the \emph{optimal} dependence on $\alpha$: \cref{thm:oblivious-sparse-regression} achieves consistency as long as the  fraction of inliers satisfies $\alpha=\omega(\sqrt{k\log d/n)}$. To get an intuition, observe that lower bounds for standard sparse regression show that, already for $\bm \eta\sim N(0, \sigma\cdot \Id_n)$, it is possible to achieve consistency only for $n= \omega(\sigma^2 k\log d)$ (if $k\le d^{1-\Omega(1)}$). As for this $\bm \eta$, the number of entries of magnitude at most $1$ is $O(n/\sigma)$ with high probability, it follows that for $\alpha=\Theta\Paren{1/\sigma}\leq O\Paren{\sqrt{(k\log d)/n}}$, no estimator is consistent.

To the best of our knowledge, \cref{thm:oblivious-sparse-regression} is the first result to achieve consistency  under such minimalistic noise settings and deterministic designs.
Previous results~\cite{DBLP:conf/nips/Bhatia0KK17, SuggalaBR019, ICML-linear-regression} focused on simpler settings of  Gaussian design $\bm X$ and deterministic noise, and provide no guarantees for more general models. The techniques for \cref{thm:oblivious-sparse-regression} also extend  to this case.

\begin{theorem}
	\label{thm:oblivious-sparse-regression-gaussian-design}
	Let $\beta^* \in \R^d$ be an unknown $k$-sparse vector and let $\bm X$ be a $n$-by-$d$ random matrix with \iid rows $\bm X_1,\ldots \bm X_n \sim N(0, \Sigma)$ for a positive definite matrix $\Sigma$.
	Further, let $\eta\in \R^n$ be a deterministic vector with $\alpha\cdot n$ coordinates bounded by $1$ in absolute value.	
	Suppose that $n\gtrsim \frac{\nu(\Sigma)\cdot k \log d}{\sigma_{\min}(\Sigma) \cdot \alpha^2}$,
	where $\nu(\Sigma)$ is the maximum diagonal entry of $\Sigma$ and $\sigma_{\min}(\Sigma)$ is its smallest eigenvalue.
	
	Then, with probability at least $1-d^{-10}$ over $\bm X$, given $\bm X$ and $\bm y=\bm X\beta^*+\eta$, the 
	estimator \cref{eq:results-estimator-regression} satisfies
	\begin{equation*}
	\frac{1}{n} \Norm{\bm X\Paren{ \hat{\bm  \beta}- \beta^*}}^2 \leq O\Paren{\frac{\nu(\Sigma) \cdot k\log d}{\sigma_{\min}(\Sigma) \cdot \alpha^2 \cdot n}} 
		\mbox{\qquad and\qquad} 
	\Norm{{ \hat{\bm  \beta}- \beta^*}}^2 \leq O\Paren{\frac{\nu(\Sigma) \cdot k\log d}{\sigma_{\min}^2(\Sigma) \cdot \alpha^2 \cdot n}} \,.
	\end{equation*}
\end{theorem}

Even for standard Gaussian design $X\sim N(0,1)^{n\times d}$, the above theorem improves over previous results, which required sub-optimal sample complexity $n\gtrsim\Paren{k/\alpha^2}\cdot\log d \cdot \log \Norm{\beta^*}$. 
For non-spherical Gaussian designs, the improvement over state of the art \cite{SuggalaBR019} is more serious: their algorithm requires $\alpha \ge \Omega\Paren{1/\log\log n}$, while our \cref{thm:oblivious-sparse-regression-gaussian-design} doesn't have such restrictions and works for all $\alpha \gtrsim \sqrt{\frac{\nu(\Sigma)}{\sigma_{\min}(\Sigma)}\cdot {\frac{k\log d}{ n}}}$; in many interesting regimes $\alpha$ is allowed to be smaller than $n^{-\Omega(1)}$. The dependence on $\alpha$ is nearly optimal: the estimator is consistent as long as  $\alpha \ge \omega\Paren{\sqrt{\frac{\nu(\Sigma)}{\sigma_{\min}(\Sigma)}\cdot {\frac{k\log d}{ n}}}}$, and from the discussion after \cref{thm:oblivious-sparse-regression-gaussian-design}, if $\alpha \le O\paren{\sqrt{(k\log d)/n}}$, no estimator is consistent.

Note that while we can deal with general covariance matrices, to compare \cref{thm:oblivious-sparse-regression-gaussian-design} with \cref{thm:oblivious-sparse-regression} it is easier to consider $\Sigma$ in a normalized form, when $\nu(\Sigma)\le 1$. This can be easily achieved by scaling $\bm X$.
Also note that \cref{thm:oblivious-sparse-regression} can be generalized to the case $\norm{X^i}\le  \sqrt{\nu n}$ for arbitrary $\nu > 0$, and then the error bounds and the bound on $m$ should be multiplied by $\nu$.

The RE-property of \cref{thm:oblivious-sparse-regression} is a standard assumption in  sparse regression and is satisfied by a large family of matrices. For example, with high probability a random matrix $\bm X$ with \iid rows sampled from $\cN(0, \Sigma)$, with positive definite $\Sigma\in \R^{d\times d}$ whose diagonal entries are bounded by $1$, satisfies the RE-property with parameter $\Omega(\sigma_{\min}(\Sigma))$ for all subsets of $[d]$ of size $k$ (so for every possible support of $\beta^*$) as long as long as $n\gtrsim \frac{1}{\sigma_{\min}\Paren{\Sigma}}\cdot k\log d$ (see \cite{wainwright_2019}, section 7.3.3). 
The well-spread assumption is satisfied for such $\bm X$ with for all sets $S\subset [n]$ of size $m\le cn$ (for sufficiently small $c$)  and for all subsets of $[d]$ of size $k$ as long as $n\gtrsim \frac{1}{\sigma_{\min}\Paren{\Sigma}}\cdot k\log d$.

\subsection{Optimal fraction of inliers for principal component analysis under oblivious noise}\label{sub:result-lower_bound}

We  show here that the dependence on $\alpha$ we obtain in \cref{thm:oblivious-pca} is information theroetically optimal up to constant factors.
Concretely, let $L^\ast,\bm N,\bm Y,\alpha,\rho$ and $\zeta$ be as in \cref{thm:oblivious-pca}, and let $0<\epsilon<1$ and $0<\delta<1$. A successful $(\epsilon,\delta)$-weak recovery algorithm for PCA is an algorithm that takes $\bm Y$ as input and returns a matrix $\hat{\bm L}$ such that $\Normf{\hat{\bm L}-L^*}\leq \epsilon\cdot \rho$ with probability at least $1-\delta$. 

It can be easily seen that the Huber-loss estimator of \cref{thm:oblivious-pca} fails to be a successful weak-recovery algorithm if $\alpha=o\paren{\sqrt{r/n}}$ (for both cases $\zeta\leq \rho/n$ and $\rho/n\leq \zeta$, we need $\alpha=\Omega(\sqrt{r/n})$.) A natural question to ask is whether the condition $\alpha=\Omega\paren{\sqrt{r/n}}$ is necessary in general. The following theorem shows that if $\displaystyle\alpha=o\paren{\sqrt{r/n}}$, then weak-recovery is information-theoretically impossible. This means that the (polynomially small) fraction of inliers that the Huber-loss estimator of \cref{thm:oblivious-pca} can deal with is optimal up to a constant factor.
\begin{theorem}
	\label{thm:IT_lower_bound}
	There exists a universal constant $C_0>0$ such that for every $0<\epsilon<1$ and $0<\delta<1$, if $\alpha:=\min_{i,j\in[n]}\mathbb{P}[|\bm N_{i,j}|\leq\zeta]$ satisfies $\alpha< C_0\cdot (1-\epsilon^2)^2\cdot(1-\delta)\cdot\sqrt{r/n}$,
	and $n$ is large enough, then it is information-theoretically impossible to have a successful $(\epsilon,\delta)$-weak recovery algorithm. 
	
	The problem remains information-theoretically impossible (for the same regime of parameters) even if we assume that $L^*$ is incoherent; more precisely, even if we know that $L^*$ has incoherence parameters that are as good as those of a random flat matrix of rank $r$, the theorem still holds.
\end{theorem}

\section*{Notation and outline}\label{sec:preliminaries}

\paragraph{Notation}
For $\beta\in \R^d$ we define the function $\norm{\beta}_0:=\sum_{i \in [d]}\ind{\beta_i\neq 0}\,.$ For a subspace $\Omega\subseteq\R^d$ we denote the projection of $\beta$ onto $\Omega$ by $\beta_{\Omega}$. We write $\Omega^\bot$ for the orthogonal complement of $\Omega$. For $N\in \N$ we denote $[N] := \Set{1,2,\ldots,N}$. We write $\log$ for the logarithm to the base $e$. 

For a matrix $X\in \R^{d\times d}$ we denote by $\rspan{X}$ and $\cspan{X}$ respectively the rows and columns span of $X$, and we write $\Norm{X}$ for the spectral norm of $X$, $\Normf{X}$ for its Frobenius norm, $\Normn{X}$ for its nuclear norm and  
$\Normm{X}:=\max_{i,j \in[n]}\Abs{X_\ij}$.  For a vector $v\in \R^N$ we write $\norm{v}$ for its Euclidean norm, $\normo{v} = \sum_{i=1}^N\abs{v_i}$ and $\normi{v_i} = \max_{i\in[N]}\abs{v_i}$.
For a norm $\Norm{\cdot}$ we write $\Norm{\cdot}^*$ for its dual. 
We denote by $\bm G\sim N(0,1)^{n\times d}$ a random $n$-by-$d$ matrix $\bm G$ with \iid standard Gaussian entries. Similarly, we denote by $\bm g\sim N(0,1)^n$ an $n$-dimensional random vector $\bm g$ with \iid standard Gaussian entries.

For a set $\cS$ and a metric $\rho:\cS\times\cS\rightarrow[0,\infty)$, we denote an  {$\eps$-net} in $\cS$ by $\cN_{\eps, \rho}(\cS)$. That is,  $\cN_{\eps, \rho}(\cS)$ is a subset of $\cS$ such that for any $u \in \cS$ there exists $v\in \cN_{\eps, \rho}(\cS)$ satisfying $\rho\paren{u,v}\leq \eps$.

\paragraph{Outline}
In \cref{section:techniques}, we present the main ideas behind our results.
In \cref{sec:m-estimators}, we introduce a general framework that allows us to prove our main theorems about consistent estimation (\cref{thm:oblivious-pca}, \cref{thm:oblivious-sparse-regression} and \cref{thm:oblivious-sparse-regression-gaussian-design}). 

Sections \ref{sec:oblivious-pca} to \ref{sec:lower-bound-pca} are devoted to the proofs of theorems from \cref{sec:results}. Concretely,  in \cref{sec:oblivious-pca} we prove \cref{thm:oblivious-pca}, \cref{sec:oblivious-sparse-regression} contains a proof of \cref{thm:oblivious-sparse-regression} and in \cref{sec:oblivious-gaussian-design} we prove \cref{thm:oblivious-sparse-regression-gaussian-design}.
Complementing these algorithmic results, we prove a lower bound (\cref{thm:IT_lower_bound}) in \cref{sec:lower-bound-pca}.

Finally, \cref{sec:additional-tools} contains the proofs of a few facts about the Huber loss and \cref{sec:random-matrices-bounds} contains a few facts from probability theory.
	
\section{Techniques}\label{section:techniques}

To illustrate our techniques in proving statistical guarantees for the Huber-loss estimator, we first use sparse linear regression as a running example. Then, we discuss how the same ideas apply to principal component analysis. Finally, we also remark our techniques for lower bounds.

\subsection{Sparse linear regression under oblivious noise}\label{sec:techniques-regression}

We consider the model of \cref{thm:oblivious-sparse-regression}. Our starting point to attain the guarantees for our estimator \cref{eq:results-estimator-regression}, i.e., $\hat{\bm  \beta} \coloneqq \arg \min_{\beta\in \R^d} \bm F_2(\beta)+ 100\sqrt{n\log d}\Norm{\beta}_1$, is a classical approach for $M$-estimators (see e.g. \cite{wainwright_2019}, chapter 9). 
For simplicity, we will refer to $\bm F_2(\beta)$ as the loss function and to $\Norm{\beta}_1$ as the regularizer.
At a high level, it consists of the following two ingredients:

\begin{enumerate}[(I)]
\item an upper bound on some norm of the gradient of the loss function at the parameter $\beta^*$, \label{enum_one}
\item a lower bound on the curvature of the loss function (in form of a local strong convexity bound) within a structured neighborhood of $\beta^*$. The structure of this neighborhood can roughly be controlled by choosing the appropriate regularizer. \label{enum_two}
\end{enumerate}

The key aspect of this strategy is that the strength of the statistical guarantees of the estimator crucially depends on the \textit{directions} and the \textit{radius} in which we can establish lower bounds on the curvature of the function.
Since these features are inherently dependent on the landscape of the loss function and the regularizer, they may differ significantly from problem to problem. 
This strategy has been applied successfully for many related problems such as compressed sensing or matrix completion albeit with standard noise assumptions.\footnote{The term "standard noise assumptions" is deliberately vague; as a concrete example, we will refer to  (sub)-Gaussian noise distributions.  See again chapter 9 of \cite{wainwright_2019} for a survey.}
Under oblivious noise, \cite{ICML-linear-regression} used a particular instantiation  of this framework to prove optimal convergence of the Huber-loss -- without any regularizer -- for standard linear regression.
Such estimator, however, doesn't impose any structure on the neighborhood of $\beta^*$ considered in (\ref*{enum_two}) and thus, can only be used to obtain sub-optimal guarantees for sparse regression.

In the context of sparse regression, the above two conditions translate to:
(\ref*{enum_one}) an upper bound on the largest entry in absolute value of the gradient of the loss function at $\beta^*$, and 
(\ref*{enum_two}) a lower bound on the curvature of $\bm F_2$ within the set of 
\textit{approximately} $k$-sparse vectors\footnote{We will clarify this notion in the subsequent paragraphs.} close to $\beta^*$. 
We use this recipe to show that \textit{all} approximate minimizers of $\bm F$ are close to $\beta^*$.
While the idea of restricting to only approximately sparse directions has also been applied for the LASSO estimator in sparse regression under standard (sub)-Gaussian noise, in the presence of oblivious noise, our analysis of the Huber-loss function requires a more careful approach.

More precisely, under  the assumptions of \cref{thm:oblivious-sparse-regression}, the error bound can be computed as
\begin{equation}\label{eq:bound-from-parameters}
O\Paren{\frac{s\Normreg{G}^*}{\kappa}}\,,
\end{equation}
where $G$ is a gradient of Huber loss at $\beta^*$, $\normreg{\cdot}^*$ is a norm dual to the regularization norm (which is equal to $\norm{\cdot}_{\max}$ for $\ell_1$ regularizer), $s$ is a \emph{structure} parameter, which is equal to $\sqrt{k/\lambda}$, and $\kappa$ is a restricted strong convexity parameter. Note that by error here we mean $\frac{1}{\sqrt{n}}\norm{X\paren{ \hat{\bm  \beta}- \beta^*}}$.

Similarly, under under  assumptions of \cref{thm:oblivious-pca}, we get the error bound \cref{eq:bound-from-parameters}, where  $G$ is a gradient of Huber loss at $L^*$, $\normreg{\cdot}^*$ is a norm dual to the nuclear norm (i.e. the spectral norm),  structure parameter $s$ is $\sqrt{r}$, and $\kappa$ is a restricted strong convexity parameter. 
For more details on the conditions of the error bound, see the supplementary material. Below, we explore the bounds on the norm of the gradient and on the restricted strong convexity parameter.

\paragraph{Bounding the gradient of the Huber loss}
The gradient of the Huber-loss $\bm F_2(\cdot)$ at $\beta^*$ has the form $\nabla \bm F_2(\beta^*) = \sum_{i=1}^{n} f'_2\Brac{\bm \eta_i}\cdot X_i$,
where $X_i$ is the $i$-th row of $X$. The random variables $ f'_2\Brac{\bm \eta_i}$, $i\in [n]$, are independent, centered,  symmetric and bounded by $2$.
Since we assume that each column of $X$ has norm at most $\sqrt{n}$, the entries of the row $X_i$ are easily bounded by $\sqrt{n}$.
Thus, $\nabla \bm F_2(\beta^*) $ is a vector with independent, symmetric entries with bounded variance, so its behavior can be easily studied through standard concentration bounds. 
In particular, by a simple application of Hoeffding's inequality, we obtain, with high probability,
\begin{equation}\label{eq:techniques-concentration-gradient}
	\Normm{\nabla \bm F_2(\beta^*) } = \max_{j \in [d]} \BIGabs{\underset{i \in [n]}{\sum} f'_2\Brac{\bm \eta_i}\cdot  X_{\ij}}
	\leq {O}\Paren{\sqrt{n\log d}}\,.
\end{equation}

\paragraph{Local strong convexity of the Huber loss}
Proving local strong convexity presents additional challenges.
Without the sparsity constraint, \cite{ICML-linear-regression} showed that under a slightly stronger spreadness assumptions than \cref{thm:oblivious-sparse-regression}, the Huber loss is locally strongly convex within a constant radius $R$ ball centered at $\beta^*$ whenever $n\gtrsim d/\alpha^2$. (This function is not globally strongly convex due to its linear parts.)
Using their result as a black-box, one can obtain the error guarantees of \cref{thm:oblivious-sparse-regression}, but with suboptimal sample complexity.
The issue is that with the substantially smaller sample size of $n\geq \tilde{O}(k/\alpha^2)$ that resembles the usual considerations in the context of sparse regression, the Huber loss is \textit{not} locally strongly convex around $\beta^*$ uniformly across \textit{all directions}, so we cannot hope to prove convergence with optimal sample complexity using this argument.
To overcome this obstacle, we make use of the framework of M-estimators:
Since we consider a regularized version of the Huber loss, it will be enough to show local strong convexity in a radius $R$ uniformly across all directions which are \textit{approximately} $k$-sparse.
For this substantially weaker condition, $\tilde{O}(k/\alpha^2)$ will be enough.

More in details, for observations $\bm y= X\beta^*+\bm \eta$ and an arbitrary $u\in \R^d$ of norm $\norm{u}\leq R$, it is possible to lower bound the  Hessian\footnote{The Hessian does not exist everywhere. Nevertheless, the second derivative of the penalty function $f_2$ exists as an $L_1$ function in the sense that $f'_2(b)-f'_2(a)=2\int_{a}^{b}\ind{\Abs{t}\leq 2}dt$. This property is enough for our purposes.} of the Huber loss at $\beta^*+u$ by:\footnote{A more extensive explanation of the first part of this analysis can be found in \cite{ICML-linear-regression}.}
\begin{align*}
 H\bm F_2(\beta^*+u) 
 &= 
 \sum_{i=1}^n \bm f''_2\Brac{(Xu)_i-\bm \eta_i}\cdot X_i \transpose{X_i}
 \\&= \sum_{i=1}^n  \ind{\Abs{(Xu)_i-\bm \eta_i}\leq 2} \cdot X_i \transpose{X_i}
 \\&\sge M(u) := \sum_{i=1}^n \ind{\Abs{\iprod{X_i,u}}\leq 1} \cdot \ind{\Abs{\bm \eta_i}\leq 1} \cdot X_i \transpose{X_i}
\end{align*}

As can be observed, we do not attempt to exploit cancellations between $Xu$ and $\bm \eta$.
Let $\bm Q:=\Set{i\in [n]\suchthat \Abs{\bm \eta_i}\leq 1}$ be the set of uncorrupted entries of $\bm \eta$.
Given that with high probability $\bm Q$ has size $\Omega(\alpha\cdot n)$, the best outcome we can hope for is to provide a lower bound of the form $\iprod{u,M(u)u}\geq \alpha n$ in the direction $\hat{\bm \beta}-\beta^*\,.$
In \cite{ICML-linear-regression}, it was shown that if the span of the measurement matrix $X$ is well spread, then $\iprod{u, M(u)u}\geq \Omega(\alpha\cdot n)$.

If the direction $\hat{\bm \beta}-\beta^*$ was fixed, it would suffice to show the curvature in that single direction through the above reasoning.
However, $\hat{\bm \beta}$ depends on the unknown random noise vector $\bm \eta$ . 
Without the regularizer in \cref{eq:results-estimator-regression}, this dependence indicates that the vector $\hat{\bm \beta}$ may take any possible direction, so one needs to ensure local strong convexity to hold in a constant-radius ball centered at $\beta^*$.
That is, $\min_{\norm{u}\leq R}\lambda_{\min} (M(u))\geq \Omega(\alpha \cdot n)\,.$
It can be shown through a covering argument (of the ball) that this bound holds true for $ n\geq d/\alpha^2$. This is the approach of \cite{ICML-linear-regression}.

\paragraph{The minimizer of the Huber loss follows a sparse direction}
The main issue with the above approach is that no information concerning the direction $\hat{\bm \beta}-\beta^*$ is used.
In the settings of sparse regression, however, our estimator contains the regularizer $\normo{\beta}$. 
The main consequence of  the regularizer is that the \textit{direction} $\hat{\bm \beta}-\beta^*$ is approximately flat in the sense 
$\normo{\hat{\bm \beta}-\beta^*}\le O\Paren{\sqrt{k}\norm{\hat{\bm \beta}-\beta^*}}\,.$
The reason\footnote{This phenomenon is a consequence of the \textit{decomposability} of the $L_1$ norm, see the supplementary material.} is that due to the structure of the objective function in \cref{eq:results-estimator-regression} and concentration of the gradient \cref{eq:techniques-concentration-gradient}, the penalty for dense vectors is larger than the absolute value of the inner product $\iprod{\nabla\bm F(\beta^*), \hat{\bm \beta}-\beta^*}$ (which, as previously argued, concentrates around its (zero) expectation).

This specific structure of the minimizer implies that it suffices to prove local strong convexity \textit{only} in  approximately sparse directions. For these set of directions, we carefully construct a sufficiently small covering set  so that  $n\geq \tilde{O}(k/\alpha^2)$ samples suffice to ensure local strong convexity over it.

\begin{remark}[Comparison with LASSO]
	it is important to remark that while this approach of only considering approximately sparse directions has also been used in the context of sparse regression under Gaussian noise (e.g. the LASSO estimator), obtaining the desired lower bound is considerably easier in these settings as it directly follows from the restricted eigenvalue property of the design matrix.
	In our case, we require an additional careful probabilistic analysis which uses a covering argument for the set of approximately sparse vectors.
	As we see however, it turns out that we do not need any additional assumptions on the design matrix when compared with the LASSO estimator except for the well-spreadness property (recall  that some weak version of well-spreadness is indeed necessary in robust settings, see \cite{ICML-linear-regression}).
\end{remark}

\subsection{Principal component analysis under oblivious noise}\label{sec:techniques-pca}
A convenient feature of the approach in \cref{sec:techniques-regression} for sparse regression, is that it can be easily applied to additional problems. We briefly explain here how to apply it for principal component analysis. 
We consider the model defined in \cref{thm:oblivious-pca}.
We use an estimator based on the Huber loss equipped with the nuclear norm as a regularizer to enforce the low-rank structure in our estimator 
\begin{align}
	\hat{\bm L}\coloneqq  \qquad  \argminlimitversion_{\mathclap{\substack{L\in\R^{n\times n},\; \Normm{L}\leq \rho/n}}} \qquad \Paren{F_{\zeta+\rho/n}(\bm Y- L)+ 100\sqrt{n}\Paren{\zeta+\rho/n} \Normn{L}}\,.
\end{align}
In this setting, the gradient $\nabla F_{\zeta+\rho/n}(\bm Y- L^*)$ is a matrix with independent, symmetric entries which are bounded (by $\zeta+\rho/n$) and hence its spectral norm is $O\Paren{(\zeta+\rho/n)\sqrt{n}}$ with high probability.
Local strong convexity can be obtained in a similar fashion as shown in \cref{sec:techniques-regression}: due to the choice of the Huber transition point all entries with small noise are in the quadratic part of $F$. Moreover, the nuclear norm regularizer ensures that the minimizer is an approximately low-rank matrix in the sense that $\normn{M}\le O\Paren{\sqrt{r}\Normf{M}}$. So again, it suffices to provide curvature of the loss function only on these subset of structured directions.

\begin{remark}[Incoherence vs. spikiness]\label{remark:incoherence}
Recall the discussion on incoherence in \cref{sec:results}.
If for every $\mu \le n/\log^2 n$ matrix $L^*$ doesn't satisfy $\mu$-incoherence conditions, the results in \cite{DBLP:journals/jacm/CandesLMW11,DBLP:conf/isit/ZhouLWCM10} cannot be applied. 
However, our estimator achieves error $\normf{\hat{\bm L}-L^*}/ \Normf{L^*}\to 0$ as $n\to \infty$. 
Indeed,let $\omega(1)\le f(n)\le o(\log^2 n)$ and assume $\zeta = 0$. Let $u\in \R^n$ be an $f(n)$-sparse unit vector whose nonzero entries are equal to $1/\sqrt{f(n)}$. Let $v\in \R^n$ be a vector with all entries equal to $1/\sqrt{n}$. 
Then, $u\transpose{v}$ does not satisfy incoherence with any $\mu < n/f(n)$. 
We have $\Normf{u\transpose{v}} = 1$, and the error of our estimator is $O\paren{1/(\alpha \sqrt{f(n)})}$, so it tends to zero for constant (or even some subconstant) $\alpha$.
	
	Furthermore notice that the dependence of the error of \cref{thm:oblivious-pca} on the maximal entry of $L^*$ is inherent if we do not require incoherence. Indeed, consider  $L_1=b\cdot e_1\transpose{e_1}$ for large enough $b > 0$ and $L_2 = e_2\transpose{e_2}$. For constant $\alpha$, let $\abs{N_{ij}}$ be $1$ with probability $\alpha/2$, $0$ with probability $\alpha/2$ and $b$ with probability $1-\alpha$. Then given $Y$ we cannot even distinguish between cases $L^*=L_1$ or $L^*=L_2$, and since $\Normf{L_1-L_2}\ge b$, the error should also depend on $b$.
\end{remark}

\begin{remark}[$\alpha$ vs $\alpha^2$: what if one knows which entries are corrupted?]
	As was observed in \cref{sec:results}, the error bound of our estimator is worse than the error for matrix completion by a factor $1/\sqrt{\alpha}$. We observe similar effect in linear regression: if, as in matrix completion, we are given a randomly chosen $\alpha$ fraction of observations $\Set{\Paren{X_i, y_i = \iprod{X_i, \beta^*} + \bm \eta_i}}_{i=1}^n$ where $\bm\eta\sim N(0,1)^n$, and since for the remaining  samples we may not assume any bound on the signal-to-noise ratio, then this problem is essentialy the same as linear regression with $\alpha n$ observations. Thus the optimal prediction error rate is $\Theta\paren{\sqrt{d/(\alpha n)}}$. Now, if we have $y = X\beta^* +\bm\eta$, where $\bm\eta \sim N(0,1/\alpha^2)$, then with probability $\Theta(\alpha)$, $\abs{\eta_i}\le 1$, but the optimal prediction error rate in this case is $\Theta\paren{\sqrt{d/(\alpha^2 n)}}$. So in both linear regression and robust PCA, \textit{prior knowledge} of the set of corrupted entries makes the problem easier.
\end{remark}

\subsection{Optimal fraction of inliers for principal component analysis under oblivious noise}
In order to prove \cref{thm:IT_lower_bound}, we will adopt a generative model for the hidden matrix $\bm L^*$: We will generate $\bm L^*$ randomly but assume that the distribution is known to the algorithm. This makes the problem easier. Therefore, any impossibility result for this generative model would imply impossibility for the more restrictive model for which $L^*$ is deterministic but unknown.

We generate a random flat matrix $\bm L^*$ using $n\cdot r$ independent and uniform random bits in such a way that $\bm L^*$ is of rank $r$ and incoherent with high probability. Then, for every constant $0<\xi<1$, we find a distribution for the random noise $\bm N$ in such a way that the fraction of inliers satisfies $\alpha:=\mathbb{P}[|\bm N_{ij}|\leq \zeta]=\Theta\left(\xi\sqrt{r/n}\right)$, and such that the mutual information between $\bm L^*$ and $\bm Y=\bm L^*+\bm N$ can be upper bounded as $I(\bm L^*;\bm Y)\leq O(\xi\cdot n\cdot r)$. Roughly speaking the smaller $\xi$ gets, the more independent $\bm L^*$ and $\bm Y$ will be. Now using an inequality that is similar to the standard Fano-inequality but adapted to weak-recovery, we show that if there is a successful $(\epsilon,\delta)$-weak recovery algorithm for $\bm L^*$ and $\bm N$, then $I(\bm L^*;\bm Y)\geq \Omega\left((1-\epsilon^2)^2\cdot(1-\delta)\cdot n\cdot r\right)$. By combining all these observations together, we can deduce that if $\xi$ is small enough, it is impossible to have a successful $(\epsilon,\delta)$-weak recovery algorithm for $\bm L^*$ and $\bm N$.
	\section{Meta-Theorem}\label{sec:m-estimators}

We present a high-level theorem which will be applied to prove \cref{thm:oblivious-pca}, \cref{thm:oblivious-sparse-regression} and \cref{thm:oblivious-sparse-regression-gaussian-design}.
Recall the general setting an estimation problem: 
we start with a family of probability distributions $\cP:=\Set{\bbP_\theta\suchthat \theta\in \Omega}$ over some space $\cZ$ and indexed by some parameter $\theta\in \Omega$. 
We observe a collection of $n$ \textit{independent} samples $\bm Z=\Paren{\bm Z_1,\ldots, \bm Z_n}$  taking value in $\cZ$, drawn from an unknown probability distribution $\bbP_{\theta^*}\in \cP$. 
We assume $\Omega\subseteq \R^d$ and $\cZ \subseteq \R^D$ for some integers $d$ and $D$.
Our goal is then to recover $\theta^*$.
That is, given $\bm Z$, the goal is to find $\hat{\bm \theta}\in \R^d$ such that for some suitable  error function $\cE:\R^d\rightarrow [0,\infty)$, the value $\cE\Paren{\theta^*-\hat{\bm\theta}}$ is as small as possible.
It is clear that this general setting also captures settings in which the observations are perturbed by oblivious adversarial noise. 

On a high level, we will use the following scheme:
\begin{enumerate}
	\item Let $\normreg{\cdot}:\R^d\rightarrow [0,\infty)$ be a norm, and let $\gamma\in \R$ be a scalar. 
	Design a cost function $\bm F:\R^d\rightarrow \R_{\geq 0}$ which depends on $\bm Z\,.$
	\item For a set $\cC\subseteq\R^d\,,$ show that the target parameter  (or some approximation of it) 
	\begin{align*}
		\hat{\bm  \theta}:=\arg\min_{\theta \in \cC}\Paren{\bm F\Paren{\theta}+\gamma\normreg{\theta}}
	\end{align*} satisfies $\cE\Paren{\theta^*-\hat{\bm \theta}}\leq R$ for some acceptable $R\geq 0$ with high probability over the samples $\bm Z\,.$ 
	\item Argue that $\hat{\bm  \theta}$ can be computed efficiently. 
\end{enumerate}
The norm $\normreg{\cdot}$ is often referred to as a \textit{regularizer}. Its role is to enforce a certain structure on the target parameter.
For example, in the context of sparse linear regression $\bm y = X\beta^*+\bm \eta$  with $\beta^*\in \R^d$ being a $k$-sparse vector, the LASSO estimator:
$\hat{\bm \beta}:=\arg\min_{\beta\in \R^d}\Paren{\Snorm{X\beta-\bm y}+\gamma\Normo{\beta}}$ follows the description above.
In this example, the cost function is the squared euclidean norm and the regularizer corresponds to a convex relaxation of the norm $\Norm{\beta}_0\,.$

If the cost function and the set $\cC$ are convex and satisfy mild assumptions, the estimator can be computed efficiently (in polynomial time). The estimators that we use for PCA and sparse linear regression can be computed in polynomial time. For more details on computational aspects of convex optimization, see \cite{vishnoi2018algorithms}.

 For convex cost functions the meta-theorem below   (which appears in different forms in the literature, e.g. see \cite{wainwright_2019}, section 9.4) can be used to mechanically bound the guarantees of the estimator. Before stating the theorem, let's define the following set: for a norm $\normreg{\cdot}$ and for a vector subspace $V\subseteq \R^d$ and $b\ge 1$, we denote
	\[
	\cS_b\paren{V} = \Set{u\in\R^d \suchthat \normreg{u}\le b\normreg{u^{}_V}}\,,
	\]
	where $u^{}_V$ is the orthogonal projection of $u$ on $V$.

\begin{theorem}\label{thm:meta-theorem}
	Let $\gamma,\kappa, R, s$ be positive real numbers and let $\cC \subseteq \R^d$ be a convex set. Consider a vectors space $\Omega \subseteq \R^d$ and let $\theta^* \in \Omega\cap \cC$. 
	
	Let $\normreg{\cdot}:\R^d\rightarrow [0,\infty)$ be a norm and consider a continuous error function $\cE:\R^d\rightarrow [0,\infty)$ such that $\cE(0)=0$.
	Let $F:\R^d\rightarrow \R$ be a convex differentiable cost function.
	
	Suppose that there exists a vector space $\overline{\Omega}$ such that $\Omega\subseteq \overline{\Omega}\subseteq\R^d$ and such that the following properties hold:
	\begin{enumerate}
		\item[(Decomposability)] For all $u\in \Omega$ and $v\in \overline{\Omega}^\bot$,
		\begin{align}\label{eq:decomposability}
		\normreg{v+u}=\normreg{v}+\normreg{u}\,.
		\end{align}
        \item[(Contraction)] For all $u\in \cS_4\paren{\overline{\Omega}}$,
		\begin{align}\label{eq:contraction}
		\normreg{u}\leq s \cdot \cE\Paren{u}\,.
		\end{align}
		\item[(Gradient bound)] The dual norm of $\normreg{\cdot}$ of gradient of $F$ at $\theta^*$ satisfies
		\begin{align}\label{eq:subgradient-bound}
			\normreg{\nabla F\Paren{\theta^*}}^*\leq \gamma/2\,.
		\end{align}
		\item[(Restricted local strong convexity)] 
		 Let \[\cB_R:=\Set{u \in \R^d\suchthat \cE\Paren{u} =  R\,, \theta^*+u\in \cC}\,.\] 
		 Then
		 \begin{align}\label{eq:local-strong-convexity}
		 \forall u\in \cB_R\cap \cS_4\paren{\overline{\Omega}} \qquad
			F({\theta^*+u})\geq F(\theta^*)+
			\iprod{\nabla F\Paren{\theta^*},u}+
			\frac{\kappa}{2}\Paren{\cE\Paren{u}}^2\,.
		\end{align}
		\item[(Bound on radius)]  Parameters $\gamma$,$\kappa$, $R$ and $s$ satisfy	
		\begin{align}\label{eq:bound-on-radius}
		\frac{\gamma\cdot s}{\kappa} \leq R/4\,.
		\end{align}
		\end{enumerate}
		Then, for every $\theta' \in \cC$ such that 
		$F(\theta') + \gamma \normreg{\theta'} \le F(\theta^*) + \gamma \normreg{\theta^*}\,,$
		\begin{align*}
			\cE\Paren {\theta'-\theta^*} < R\,.
		\end{align*}
\end{theorem}

For completeness, we include  a proof of \cref{thm:meta-theorem}.  We will need the following lemma.

\begin{lemma}\label{lem:regularizer-decomposability}
	Consider the settings of \cref{thm:meta-theorem}. If  $\theta\in \cC$ satisfies
	\[
	F(\theta) + \gamma \normreg{\theta} \le F(\theta^*) + \gamma \normreg{\theta^*}\,,
	\]
	then $\theta-\theta^*\in \cS_4\paren{\overline{\Omega}}$.
\end{lemma}
\begin{proof}
		Denote $\Delta=\theta-\theta^*\,.$
		By the decomposability of the regularizer \cref{eq:decomposability},
		\begin{align*}
			\Normreg{\theta^*+\Delta} &= \Normreg{\theta^*_{\Omega}+\Delta_{\overline{\Omega}}+\Delta_{\overline{\Omega}^\bot}} & \\
			&\geq \Normreg{\theta^*_{\Omega}+\Delta_{\overline{\Omega}^\bot}}-\Normreg{\Delta_{\overline{\Omega}}} & \mbox{ (Triangle Inequality)}\\
			&= \Normreg{\theta^*_{\Omega}}+\Normreg{\Delta_{\overline{\Omega}^\bot}}-\Normreg{\Delta_{\overline{\Omega}}}\,. & \mbox{ (Decomposability of } \normreg{\cdot} \mbox{ )}.
		\end{align*}
		By convexity of the cost function and \Holder's inequality,
		\[
			F\Paren{\theta^*+\Delta}-F(\theta^*)\geq -\Abs{\iprod{\nabla F\Paren{\theta^*}, \Delta}}
			\geq -\normreg{\nabla F\Paren{\theta^*}}^*\cdot \normreg{\Delta}\,.
		\]
		Hence by the gradient bound and the decomposability of the regularizer, 
		\[
		F\Paren{\theta^*+\Delta}-F(\theta^*)\geq -\frac{\gamma}{2}\cdot  \normreg{\Delta}= -\frac{\gamma}{2}\Paren{\Normreg{\Delta_{\overline{\Omega}}}+\Normreg{\Delta_{\overline{\Omega}^\bot}}}\,.
		\]
		Recall that $F(\theta^*+{\Delta})+\gamma\,\normreg{\theta^*+{\Delta}}\leq F(\theta^*)+\gamma\,\normreg{\theta^*}$, hence
		\begin{align*}
			0& \geq \gamma\Paren{\Normreg{\theta^*+{\Delta}}-\Normreg{\theta^*_{\Omega}}}+ \Paren{F\Paren{\theta^*+\Delta}-F(\theta^*)} \\
			& \geq \gamma\Paren{\Normreg{\theta^*+{\Delta}}-\Normreg{\theta^*_{\Omega}}}-\frac{\gamma}{2}\Paren{\Normreg{{\Delta}_{\overline{\Omega}}}+\Normreg{{\Delta}_{\overline{\Omega}^\bot}}}\\
			&\geq \gamma \Paren{\Normreg{{\Delta}_{\overline{\Omega}^\bot}}-\Normreg{{\Delta}_{\overline{\Omega}}}}-\frac{\gamma}{2}\Paren{\Normreg{{\Delta}_{\overline{\Omega}}}+\Normreg{{\Delta}_{\overline{\Omega}^\bot}}}\\
			&= \frac{\gamma}{2}\Paren{\Normreg{{\Delta}_{\overline{\Omega}^\bot}} -3\Normreg{{\Delta}_{\overline{\Omega}}}}\,.
		\end{align*}
		Therefore, we have $\Normreg{{\Delta}_{\overline{\Omega}^\bot}} \leq 3\Normreg{{\Delta}_{\overline{\Omega}}}$, and thus
		\[\Normreg{{\Delta}}\leq \Normreg{{\Delta}_{\overline{\Omega}^\bot}}+\Normreg{{\Delta}_{\overline{\Omega}}}\leq 4\,\Normreg{{\Delta}_{\overline{\Omega}}}.\]
\end{proof}

We are now ready to prove the theorem.

\begin{proof}[Proof of \cref{thm:meta-theorem}]
	Denote $G(\theta) = F(\theta) + \gamma \normreg{\theta}$. 
	
	Assume by contradiction that there exists $\theta'\in \cC$  such that $\cE\Paren{\theta'-\theta^*}\ge R$ and $G(\theta')\le G(\theta^*)$. 
	By continuity of $\cE$, there should exist a point $\tilde{\theta}$ on the segment between $\theta'$ and $\theta^*$ such that $\cE(\tilde{\theta}-\theta^*) = R$. Since $\cC$ is convex, $\tilde{\theta}\in \cC$, so $\tilde{\theta} - \theta^*\in \cB_R$. By convexity of $G$, $G(\tilde{\theta})\le G(\theta^*)$.
	Denote $\tilde{\Delta} = \tilde{\theta} - \theta^*$.	
	We get
	\begin{align*}
			F\Paren{\theta^*+{\tilde{\Delta}}}-F(\theta^*)&\leq \gamma\Paren{\Normreg{\theta^*}-\Normreg{{\tilde{\Delta}}+\theta^*}}  & \paren{ G(\tilde{\theta})\le G(\theta^*)}
			\\
			&\leq \gamma\cdot \Normreg{{\tilde{\Delta}}}  & (\mbox{Triangle Inequality})
			\\
			&\leq \gamma\cdot s\cdot \cE\paren{{\tilde{\Delta}}}\,. & (\mbox{\cref{lem:regularizer-decomposability}}\; \&\; \mbox{\cref{eq:contraction}})
	\end{align*}
	By restricted local strong convexity (\cref{eq:local-strong-convexity}) and the Gradient bound (\cref{eq:subgradient-bound}), we have
	\begin{align*}
	\Paren{\cE\paren{{\tilde{\Delta}}}}^2& \leq \frac{2}{\kappa}\Paren{\Abs{\iprod{\nabla F\Paren{\theta^*}, {\Delta}}}+\Paren{F\Paren{\theta^*+{\tilde{\Delta}}}-F(\theta^*)}} & (\mbox{\cref{eq:local-strong-convexity}})\\
	& \leq \frac{2}{\kappa}\Paren{\Abs{\iprod{\nabla F\Paren{\theta^*}, {\Delta}}} +{\gamma\cdot s}\cdot \cE\paren{{\tilde{\Delta}}}}&\\
	     & \leq \frac{2}{\kappa}\Paren{\normreg{\nabla F\Paren{\theta^*}}^*\normreg{{\tilde{\Delta}}}+{\gamma\cdot s}\cdot \cE\paren{{\tilde{\Delta}}}}& (\mbox{\Holder's inequality})\\
		&< 4\cdot \frac{{\gamma\cdot  s}\cdot \cE\paren{{\tilde{\Delta}}}}{\kappa}& (\mbox{\cref{eq:subgradient-bound}  \& \cref{eq:contraction}})\\
		& \le R \cdot \cE\paren{{\tilde{\Delta}}}\,. & (\mbox{\cref{eq:bound-on-radius}})
	\end{align*}
	So $\cE\paren{\tilde{\Delta}} < R$, leading to a contradiction. 
	 Hence every $\theta'\in \cC$ such that $G(\theta')\le G(\theta^*)$ satisfies $\cE\Paren{\theta'-\theta^*}< R$.
\end{proof}

\section{Principal component analysis with oblivious outliers (\cref{thm:oblivious-pca})}
\label{sec:oblivious-pca}
We will prove \cref{thm:oblivious-pca}, that we restate in this section

Recall that for $L\in \R^{n\times n}$, $F_h(L) = \sum_{i,j\in[n]} f_h\Paren{L_{ij}}$, where
\[
f_h(t):=\begin{cases}
\frac{1}{2}t^2&\text{for }\abs{t}\leq h\,,\\
h(\abs{t}-\frac{h}{2})& \text{otherwise}.
\end{cases}
\]
\begin{theorem*}[Restatement of \cref{thm:oblivious-pca}]
	Let $L^*\in \R^{n\times n}$ be an unknown deterministic matrix, let $\bm N^*\in \R^{n\times n}$ be a random matrix with independent, symmetrically distributed (about zero) entries and let $\alpha:=\min_{i,j \in [n]}\bbP \Set{\Abs{\bm N_\ij}\leq \zeta}$ for some $\zeta \ge 0$. 
	Suppose that $\rank(L^*)=r$ and $\Normm{L^*}\leq \rho/n$.  
	
	Consider the following estimator:
	\begin{equation}\label{eq:huber-loss-pca-technical}
	\hat{\bm L} \coloneqq \qquad    \argminlimitversion_{\mathclap{\substack{L\in\R^{n\times n},\; \Normm{L}\leq \rho/n}}} \qquad  \Paren{F_h(\bm Y- L)+ \gamma \Normn{L}},
	\end{equation}
	where $h=\zeta + \rho/n$ and $\gamma=100\sqrt{n}\Paren{\zeta+\rho/n}$.
	
	Then, with probability at least $1-2^{-n}$ over $\bm N$, given $\bm Y = L^*+\bm N$, $\zeta$ and $\rho$, the estimator $\hat{L}$	satisfies
	\begin{align*}
	\Normf{\hat{\bm L}-L^*}\leq O\Paren{{\frac{\sqrt{rn}}{\alpha}}}\cdot(\zeta+\rho/n)\,.
	\end{align*}
\end{theorem*}

In light of  \cref{thm:meta-theorem}, we can prove \cref{thm:oblivious-pca} by showing that the estimator $\hat{\bm L}$ in \cref{eq:huber-loss-pca-technical} fulfills all the conditions of \cref{thm:meta-theorem} with
$F(L):=F_h(\bm Y-L)=F_{\zeta + \rho/n}(\bm Y-L)$,  $\normreg{\cdot}:=\normn{\cdot}$, $\gamma=100\sqrt{n}\Paren{\zeta+\rho/n}$ and $\cE(\cdot):= \Normf{\cdot}$.

To this end, we define the two vector spaces in \cref{thm:meta-theorem}, $\Omega$ and $\overline{\Omega}$, as follows:
\begin{align}
\Omega &:=\Set{L\in \R^{n\times n}\suchthat \rspan{L}\subseteq\rspan{L^*}\,, \cspan{L}\subseteq\cspan{L^*}}\,,\\
\overline{\Omega}^\bot &:=\Set{L\in \R^{n\times n}\suchthat \rspan{L} \subseteq\rspan{L^*}^\bot\,, \cspan{L}\subseteq\cspan{L^*}^\bot}\,.\label{eq:pca_omega_bot}
\end{align}
It is easy to see that $\Omega\subseteq \overline{\Omega}$ and the nuclear norm is \emph{decomposable} per \cref{eq:decomposability} with respect to $\Omega$ and $\overline{\Omega}^\bot$. That is, for all $L\in \Omega$ and $L'\in \overline{\Omega}^\bot$, we have $\Normn{L+L'}=\Normn{L}+\Normn{L'}$, satisfying condition~\cref{eq:decomposability}. 

Moreover, since $L^*$ has rank $r$, \cref{eq:pca_omega_bot} implies that any matrix in $\overline{\Omega}$ has rank at most $2r$. 
Hence, we immediately obtain that for all $L\in \cS_4\paren{\overline{\Omega}} = \Set{L \in \R^{n\times n} \suchthat 
	\Normn{L} \le 4\Normn{L_{\overline{\Omega}} } }\,,$ $\Normn{L}\leq 4\sqrt{2r}\Normf{L}$, satisfying condition~\cref{eq:contraction} with ${s=4\sqrt{2r}}$.

It remains to prove the gradient bound of the condition \cref{eq:subgradient-bound}, i.e., a bound on the spectral norm of $\nabla F_h(\bm Y- L^*)$  (since the dual norm of the nuclear norm is the spectral norm), and  the local strong convexity of the condition \cref{eq:local-strong-convexity}. 

We start with proving the gradient bound:

\begin{lemma}[Gradient bound of spectral norm]\label{lem:subgradient-bound-pca}
	Consider the settings of \cref{thm:oblivious-pca}, and let $\delta \in (0,1)$. Then with probability at least $1-\delta/2$,
	\begin{align*}
	\Norm{\nabla F_h(\bm Y-L^*)}\leq 10{h\sqrt{n+\log{(2/\delta)}}}\,.
	\end{align*}
	\begin{proof}
		By definition of the Huber penalty for all $i,j \in [n]$ 
		\begin{align*}
		-h\leq \nabla f_h(\bm Y_\ij - L^*_\ij)=\nabla f_h(\bm N_\ij)\leq h\,.
		\end{align*}
		That is, entries are independent, symmetric and bounded by $h$ in absolute value. Hence by \cref{fact:spectral-norm-bounded-entries}, with probability $1-\delta/2$ the spectral norm of this matrix is bounded by $10{h\sqrt{n+\log{(2/\delta)}}}$. 
	\end{proof}
\end{lemma}

\paragraph{Proof of local strong convexity}
We first bound the size of an \emph{$\epsilon$-net} for the set of approximately low-rank matrices (\cref{lem:eps-net-nuclear-norm}) and then apply this bound to derive a lower bound for the second-order integral of the Huber-loss function with penalty $h$ (\cref{lem:local-strong-convexity-pca}).

\begin{lemma}[$\eps$-Net for approximately low-rank matrices]\label{lem:eps-net-nuclear-norm}
	Let  $0< \eps < 1$ and $s\ge 1$. Define
	\begin{align*}
	\cL_s:=\Set{L\in \R^{n\times n}\suchthat \Normn{L}\leq s\Normf{L}\,, \Normf{L}\leq 1}\,.
	\end{align*}
	Then $\cL_s$ has an $\eps$-net of size $\exp\Brac{\frac{16s^2n}{\eps^2}}$.
	\begin{proof}
		Let $W$ be a $n$-by-$n$ random matrix with i.i.d entries $W_\ij \sim N(0,1)$.	By Sudakov's minoration \cref{fact:sudakov-minoration}, we have
		\begin{align*}
		\sqrt{\log \Abs{\cN_{\eps, \Normf{\cdot}}( \cL_s)}} &\leq \frac{2}{\eps}	\E \sup_{L \in \cL_s} \iprod{W,L} & (\mbox{\cref{fact:sudakov-minoration}})\\
		&\leq \frac{2}{\eps}\sup_{L \in \cL_s} \E \Norm{W}\cdot\Normn{L} &  (\mbox{\Holder's inequality})\\
		&\leq \frac{2s}{\eps}\E \Norm{W}& (\mbox{Definition of } \cL_s)\\
		&\leq \frac{4s\sqrt{n}}{\eps} & (\mbox{\cref{fact:spectral-norm-gaussian}})\,,
		\end{align*}
		where in the last inequality we use a bound on the expected spectral norm of a Gaussian matrix \cref{fact:spectral-norm-gaussian}.
	\end{proof}
\end{lemma}

Hence the intersection of the set $\bm \cS_4(\overline{\Omega}) = \Set{L \in \R^{n\times n} \suchthat 
	\Normn{L} \le 4\Normn{L_{\overline{\Omega}} } }$ with the ball $\Set{L\in \R^{n\times n} \suchthat \Normf{L}\le 1}$ has $\eps$-net of size $\exp\Brac{\frac{16\cdot 32\cdot n}{\eps^2}} \le \exp\Brac{\frac{600\cdot n}{\eps^2}} $.

Now we can prove the restricted local strong convexity:

\begin{lemma}[Restricted local strong convexity of Huber-loss]\label{lem:local-strong-convexity-pca}
	\newcommand{\placeholder}{R}
	Consider the settings of \cref{thm:oblivious-pca}. Let $0<\delta<1, R> 0$ and $h \ge \rho/n + \zeta$.
	
	Define 
	\begin{align*}
	\cB_R:=\Set{\Delta\in \R^{n\times n}\suchthat \Normf{\Delta}= R\,, \Normm{L^*+\Delta}\leq \rho/n}\,.
	\end{align*}  
	Suppose that 
	\[
	R \ge 2000\cdot \frac{\rho/n}{\alpha}\cdot \sqrt{rn+\log\paren{2/{\delta}}}\,.
	\]	
	Then with probability at least $1-\delta/2$, for all $\Delta \in \cB_R\cap \cS_4\paren{\overline{\Omega}}$, 
	\begin{align*}
	F_h(L^*+\Delta)\geq& F_h(L^*)+\iprod{\nabla F_h(L^*), \Delta}+0.01\cdot \alpha\cdot \Snorm{\Delta}_F\,.
	\end{align*}
	\begin{proof}
			Denote $M=\rho/n$.
		Consider $L$ such that $\Norm{L}_{\max}\le M$. Since $h\geq\zeta+M$,  by \cref{lem:second-order-behavior-huber}, 
		\begin{align*}
		& F_h(L)-F_h(L^*)-\iprod{\nabla F_h(L^*), L-L^*} &\\
		&\geq \frac{1}{2}\underset{i,j \in [n]}{\sum}(L_{ij}-L^*_{ij})^2\,\ind{\Abs{\bm L^*_{ij}}\leq h-\zeta}\cdot \ind{\Abs{L_{ij}-L^*_{ij}}\leq \zeta } & (\mbox{ \cref{lem:second-order-behavior-huber}})\\
		&= \frac{1}{2}\underset{i,j \in [n]}{\sum}(L_{ij}-L^*_{ij})^2\,\ind{\Abs{\bm N_\ij}\leq \zeta}. & (\Normm{L^*}\leq M\leq h-\zeta )
		\end{align*}
		
		We will lower bound this quantity for every $L$ such that $L-L^*\in \cB_R\cap \cS_4\paren{\overline{\Omega}}$.
		Denote $\cC_R := \cB_R\cap \cS_4\paren{\overline{\Omega}}$ and let $\Delta:=L-L^*\in \cC_R$.
		By \cref{lem:eps-net-nuclear-norm}, there exists $(\eps\cdot \placeholder)$-net $\cN_{\eps\placeholder, \Normf{\cdot }}\Paren{\cC_R}$ 
		of size at most $\exp\Brac{\frac{16\cdot 32\cdot n}{\eps^2}} \le \exp\Brac{\frac{600\cdot n}{\eps^2}} $. (recall that $s^2 = 32r$).
		Thus, we can write $\Delta\in \cC_R$ as a sum $A+B\in \R^{n\times n}$ where $A\in \cN_{\eps\placeholder, \Normf{\cdot}}\Paren{\cC_R}$ and $\Normf{B}\leq \eps\placeholder\,.$
		It follows that  
		\begin{align}
		\underset{i,j \in [n]}{\sum}\Delta^2_\ij\cdot \ind{\Abs{\bm N_\ij}\leq \zeta}&= 
		\underset{i,j \in [n]}{\sum}(A_\ij + B_\ij)^2\cdot \ind{\Abs{\bm N_\ij}\leq \zeta} \nonumber\\
		&\ge
		\frac{1}{2}\underset{i,j \in [n]}{\sum}A_\ij^2 \cdot \ind{\Abs{\bm N_\ij}\leq \zeta} 
		-\underset{i,j \in [n]}{\sum}B_\ij^2 \cdot \ind{\Abs{\bm N_\ij}\leq \zeta} \label{eq:lem-subgradient-bound-pca-1}\,.		
		\end{align}
		
		Let $\eps= \sqrt{\alpha}/4$. Then
		\begin{equation}
		\label{eq:lem-subgradient-bound-pca-2}
		{\underset{i,j \in [n]}{\sum}B_\ij^2\cdot \ind{\Abs{\bm N_\ij}\leq \zeta}}\leq \Snorm{B}_F\leq \eps^2 R^2\le \frac{\alpha\cdot  R^2}{16} \,.
		\end{equation}
		Denote $E:=\E {\underset{i,j \in [n]}{\sum}A_\ij^2\cdot \ind{\Abs{\bm N_\ij}\leq \zeta}}$. Since $A\in \cN_{\eps\placeholder, \Normf{\cdot}}\Paren{\cC_R}\subset \cC_R$, we have $\Normf{A}= R$, hence
		\begin{equation}
		\label{eq:lem-subgradient-bound-pca-3}
		E\ge \alpha \Normf{A}^2 \ge 
		\frac{\alpha \cdot R^2}{2}\,.
		\end{equation}
		
		Moreover, 
		since $\normm{A}\leq \normm{-L^*}+\normm{A+L^*}\leq 2M$ (recall $A\in \cB_R)$ and $\alpha_{ij}=\bbP \Paren{\Abs{N_\ij}\leq \zeta}$, we have 
		$\Abs{A_\ij^2\Paren{\ind{\Abs{\bm N_\ij}\leq \zeta}-\alpha_{ij}}}\leq 4M^2$, implying that
		\begin{align*}
		\underset{i,j \in[n]}{\sum} \E A^4_\ij \Paren{\ind{\Abs{\bm N_\ij}\leq \zeta}-\alpha_{ij}}^2  &  \leq 4M^2\underset{i,j \in[n]}{\sum} \E A^2_\ij \Abs{\ind{\Abs{\bm N_\ij}\leq \zeta}-\alpha_{ij}}\\
		&= 4M^2 \underset{i,j \in[n]}{\sum} \Paren{\alpha_\ij\cdot A^2_\ij\cdot 0+ (1-\alpha_\ij)\cdot A^2_\ij\cdot \alpha_\ij}\\
		& = 4M^2 \underset{i,j \in[n]}{\sum} A^2_\ij\cdot (\alpha_\ij-\alpha_\ij^2)\\
		&\leq 4 M^2E\,.
		\end{align*}
		
		Applying Bernstein's inequality (\cref{fact:bernstein}) with $t\geq 1$ we get
		\begin{align*}
		\bbP&\Paren{\Abs{\underset{i,j \in[n]}{\sum} A^2_\ij \Paren{\ind{\Abs{\bm N_\ij}\leq \zeta}-\alpha_{ij}}}\geq t\cdot 2M\cdot\sqrt{E}+t^2\cdot 4M^2}\leq 2\exp \Paren{-t^2/4}\,.
		\end{align*}
		Note that $\Card{\cN_{\eps\placeholder, \Normf{\cdot}}(\cC_R)}\leq \exp\Brac{\frac{600rn}{\eps^2}}\le \exp\Brac{\frac{10000rn}{\alpha}}$.
		Therefore, if we set $$t=\sqrt{\frac{40000rn}{\alpha}+8\log\Paren{2/\delta}}$$ and take a union bound over $\cN_{\eps\placeholder, \Normf{\cdot}}(\cC_R)$, we obtain that with probability at least $1-\delta/2$, we have
		\begin{align*}
		\Abs{\underset{i,j \in[n]}{\sum} A_\ij^2\cdot \Paren{\ind{\Abs{\bm N_\ij}\leq \zeta}-\alpha_\ij}}\leq& \;
		400M\cdot \sqrt{E}\cdot \sqrt{\frac{rn}{\alpha}+\log\Paren{2/\delta}} \\&\;+ \Paren{400M}^2\Paren{\frac{rn}{\alpha}+\log\Paren{2/\delta}}\,.
		\end{align*}
		for all $A\in \cN_{\eps\placeholder, \Normf{\cdot}}(\cC_R)$. Now since $\displaystyle E\geq \frac{\alpha R^2}{2}$ and $\displaystyle R\geq \frac{1000M}{\sqrt{\alpha}}\sqrt{\frac{rn}{\alpha}+\log\Paren{2/\delta}}$, we have
		\begin{align*}
		\sqrt{E}\geq \frac{\sqrt{\alpha} \cdot R}{\sqrt{2}}\geq 1400 M\sqrt{\frac{rn}{\alpha}+\log\Paren{2/\delta}},
		\end{align*}
		hence, with probability at least $1-\delta/2$,
		\begin{align*}
		\Abs{\underset{i,j \in[n]}{\sum} A_\ij^2\cdot \Paren{\ind{\Abs{\bm N_\ij}\leq \zeta}-\alpha_\ij}}&\leq 
		\frac{2}{7}\cdot E + \Paren{\frac{2}{7}}^2\cdot E\leq 4E \,.
		\end{align*}
		
		By combining this with \cref{eq:lem-subgradient-bound-pca-1}, \cref{eq:lem-subgradient-bound-pca-2} and \cref{eq:lem-subgradient-bound-pca-3}, we obtain that with probability at least $1-\delta$, we have
		\begin{align*}
		\underset{i,j \in [n]}{\sum}\Delta^2_\ij\cdot \ind{\Abs{\bm N_\ij}\leq 1}
		&\ge \frac{1}{2}\Paren{E - 0.4E}-\frac{\alpha\cdot R^2}{16} 
		\\&\ge 0.15{\alpha\cdot R^2} -0.0625{\alpha\cdot R^2} 
		\\&\ge 0.08\alpha R^2\,
		\end{align*}
		concluding the proof.
	\end{proof}
\end{lemma}

\paragraph{Putting everything together} We can now combine the above results with \cref{thm:meta-theorem} to prove \cref{thm:oblivious-pca}.

\begin{proof}[\textbf{Proof of \cref{thm:oblivious-pca}}]
	By \cref{lem:subgradient-bound-pca} and \cref{lem:local-strong-convexity-pca}, we can apply \cref{thm:meta-theorem} with $\gamma = 100\Paren{\zeta+\frac{\rho}{n}}\sqrt{n+\log(2/\delta)}$, $\kappa=0.01\alpha$, and $s=4\sqrt{2r}$.
	
	It follows that for 
	\begin{align*}
	R\gtrsim \Paren{\zeta+\rho/n}
	\sqrt{\frac{r\Paren{n+\log\Paren{2/\delta}}}{\alpha^2}}
	\end{align*}
	
	the estimator $\hat{\bm L}$ defined in \cref{eq:huber-loss-pca-technical} satisfies $\Normf{\hat{\bm L}-L^*}< R$ with probability at least $1-\delta$. With $\delta = 2^{-n}$ we get the desired bound.
\end{proof}
	\section{Sparse linear regression with oblivious outliers (\cref{thm:oblivious-sparse-regression})}\label{sec:oblivious-sparse-regression}
We prove \cref{thm:oblivious-sparse-regression}, which will be restated below. Before the restatement, for easier reference, we list the three assumptions in \cref{sub:result-sparse_regression} for the design matrix $X\in\R^{n\times n}$:
\begin{enumerate}
	\item 	For every column $X^i$ of $X$, $\Norm{X^i}\le \sqrt{\nu n}$.
	\item\emph{Restricted eigenvalue property (RE-property)}:
	For every vector $u \in \R^d$ such that\footnote{For a vector $v \in \R^d$ and a set $S \sse [d]$, 
		we denote by $v_S$  the restriction of $v$ to the coordinates in $S$.} 
	$\Normo{u_{\supp(\beta^*)}} \geq 0.1\cdot \Normo{u}$,
	$\frac{1}{n}\Norm{Xu}^2\geq \lambda \cdot \Norm{u}^2$ for some parameter $\lambda >0$.
	\item\emph{Well-spreadness property}: For some $m\in[n]$ and for every vector $u \in \R^d$ such that 
	$\Normo{u_{\supp(\beta^*)}} \geq 0.1\cdot \Normo{u}$ and for every subset $S\subseteq [n]$ with $\Card{S}\geq n - m$, it holds that $\Norm{(Xu)_S}\ge \frac{1}{2} \Norm{Xu}$.
\end{enumerate}

Recall that $\bm F_2(\beta) = \sum_{i=1}^n f_2\Paren{\bm y_i- \iprod{X_i, \beta}}$, where
\[
f_2(t):=\begin{cases}
\frac{1}{2}t^2&\text{for }\abs{t}\leq 2\,,\\
2\abs{t}-2& \text{otherwise}.
\end{cases}
\]

\begin{theorem}[Restatement of \cref{thm:oblivious-sparse-regression}]\label{thm:oblivious-sparse-regression-restatement}
	Let $\beta^* \in \R^d$ be an unknown $k$-sparse vector and let $X\in \R^{n\times d}$ be a deterministic matrix such that for each column $X^i$ of $X$, $\norm{X^i}\le \sqrt{ \nu n}$, satisfying the RE-property with $\lambda>0$ and well-spreadness property with $m \gtrsim \frac{ k\log d}{\lambda\cdot \alpha^2}$ (recall that $n\ge m$).
	Further, let $\bm \eta$ be an $n$-dimensional random vector with independent, symmetrically distributed (about zero) entries and $\alpha=\min_{i\in[n]}\bbP\Set{\Abs{\bm \eta_i}\leq 1}$.
	Consider the following estimator:
	\begin{align}\label{eq:huber-loss-regression-technical}
		\hat{\bm  \beta} \coloneqq \arg \min_{\beta\in \R^d} \Paren{\bm F_2(\beta)+ 100\sqrt{\nu n\log d}\cdot\Norm{\beta}_1}\,.
	\end{align}
	Then, with probability at least $1-d^{-10}$ over $\bm \eta$, given $X$ and $\bm y=X\beta^*+\bm \eta$, the estimator $\hat{\beta}$ satisfies
	\begin{equation*}
		\frac{1}{n} \Norm{X\Paren{ \hat{\bm  \beta}- \beta^*}}^2 \leq O\Paren{\frac{\nu }{\lambda} \cdot \frac{ k\log d}{ \alpha^2 \cdot n}} \mbox{\qquad and\qquad} \Norm{{ \hat{\bm  \beta}- \beta^*}}^2 \leq O\Paren{\frac{\nu}{\lambda^2} \cdot \frac{ k\log d}{ \alpha^2 \cdot n}}\,.
	\end{equation*}
\end{theorem}

We assume $d \ge 2$ since for $d=1$ \cref{thm:oblivious-sparse-regression-restatement} is trivially true (since the probability $1-d^{-10} = 0$ in this case).

As for principal component analysis (\cref{sec:oblivious-pca}), we will prove \cref{thm:oblivious-sparse-regression-restatement} by showing that the estimator \cref{eq:huber-loss-regression-technical} fulfills the conditions of \cref{thm:meta-theorem} with $F(\beta)=\bm F_2(\beta)$, $\normreg{u}=\normo{u}$, $\gamma=100\sqrt{n\log d}$ and $\cE(u)=\frac{1}{\sqrt{n}}\Norm{Xu}$.

Let $\Omega:=\{\beta\in \R^d\mid \supp{\beta}\subseteq \supp(\beta^*)\}$ and $\overline{\Omega} := \Omega$.
Clearly, for any $v\in \Omega$ and any $v'\in \overline{\Omega}^\bot$,
\begin{align*}
	\Normo{v+v'}=\Normo{v}+\Normo{v'}\,.
\end{align*}
That is, $\Normo{\cdot}$ is decomposable, satisfying condition~\cref{eq:decomposability}.

The contraction condition~\cref{eq:contraction} holds for ${s=4\sqrt{k/\lambda}}$ since for all $v\in \cS_4\paren{\overline{\Omega}} = \Set{v \suchthat \normo{v}\le 4 \normo{v^{}_{\overline{\Omega} } } }$, $\Normo{v}\leq 4\sqrt{k}\norm{v}\le 4\sqrt{k/\lambda}\cdot \frac{1}{\sqrt{n}}\Norm{Xv}$, where the last inequality comes from the RE-property.

It remains to provide a gradient bound of the form in \cref{eq:subgradient-bound} and local strong convexity in \cref{eq:local-strong-convexity}.

\begin{lemma}[Gradient bound]\label{lem:subgradient-bound-regression}
	Consider the settings of \cref{thm:oblivious-sparse-regression-restatement}. 
	Then, with probability at least $1-\delta/2$,
	\begin{align*}
		\Normm{\nabla \bm F_2(\beta^*)}\leq 20\sqrt{\nu \cdot n\cdot \Paren{ \log d + \log(2/\delta)}}\,.
	\end{align*}
	\begin{proof}
		By definition of $f_2$,
		\begin{align*}
			 \nabla \Paren{\sum_{i=1}^n f_2(\bm y_i - \iprod{X_i,\beta^*})}= \bm z^\top X
		\end{align*}
		where $\bm z$ is a $n$-dimensional random vector with independent, symmetric entries $f'_2(\bm \eta_i)$ bounded by ${2}$ in absolute value.
		By Hoeffding's inequality (\cref{fact:hoeffding}), for $t\geq0$,
		\begin{align*}
			\bbP \Paren{\Abs{\iprod{\bm z, X_{i}}} \geq 10t\cdot 2\cdot \Norm{X_i}}\leq \exp\Paren{-t^2}\,.
		\end{align*}
		Since $\Norm{X_i}\le \sqrt{\nu n}$, taking a union bound over all $j\in[d]$ yields the statement.
	\end{proof}
\end{lemma}

\paragraph{Proof of local strong convexity}
We first bound the size of an $\epsilon$-net for the set of approximately sparse vectors (\cref{lem:eps-net-l1-norm}) and then prove the required local strong convexity bound (\cref{lem:local-strong-convexity-sparse-regression}).

\begin{lemma}[$\eps$-Net for approximately sparse vectors]\label{lem:eps-net-l1-norm}
	Let $0< \eps < 1$ and 
	\begin{align*}
		\cU_{s}:=\Set{\beta\in \R^{d}\suchthat \Normo{\beta}\leq s\cdot \frac{1}{\sqrt{n}}\Norm{X\beta}\,,
			 \frac{1}{\sqrt{n}}\Norm{X\beta}\leq 1}\,.
	\end{align*}
	Then $\cU_{s}$ has an $\eps$-net of size $\exp\Brac{\frac{16s^2\nu\log d}{\eps^2}}$ 
	in terms of distance  ${\rho(\beta,\beta')}:=\frac{1}{\sqrt{n}}\Norm{X\Paren{\beta-\beta'}}$.
	\begin{proof}
		Let $\bm w$ be an $n$-dimensional random Gaussian vector $\bm w\sim N(0,\Id_n)$.	By Sudakov's minoration (\cref{fact:sudakov-minoration}), for 
		\begin{align*}
			\sqrt{\log \Abs{\cN_{\eps, \rho(\cdot,\cdot)}( \cU_{s})}} &\leq \frac{2}{\eps}	\E \frac{1}{\sqrt{n}}\sup_{\beta \in \cU_{s}} \iprod{\bm w,X\beta} & (\mbox{\cref{fact:sudakov-minoration}})\\
			&=	 \frac{2}{\eps}	\E \frac{1}{\sqrt{n}}\sup_{\beta \in \cU_{s}} \iprod{\transpose{X}\bm w,\beta} &  \\
			&\leq \frac{2}{\eps}\E \frac{1}{\sqrt{n}}\sup_{\beta \in \cU_{s}} \Normm{\transpose{X}\bm w}\Normo{\beta}& (\mbox{\Holder's inequality})\\
			&\leq \frac{2s}{\eps}\E \frac{1}{\sqrt{n}}\Normm{\transpose{X}\bm w}& (\mbox{Definition of } \cU_{s}),\\
			&\leq \frac{4s \sqrt{\nu\log d}}{\eps}\,,
		\end{align*}
		where in the last inequality we use the bound on the expected maximal entry of a vector with $\nu$-subgaussian entries \cref{fact:max-norm-subgaussian-entries}.
	\end{proof}
\end{lemma}

\begin{lemma}[Restricted local strong convexity of Huber-loss]\label{lem:local-strong-convexity-sparse-regression}
	\newcommand{\placeholder}{2^i}
	Consider the settings of \cref{thm:oblivious-sparse-regression-restatement}. Let $0<\delta<1, R> 0$. Define 
	\begin{align*}
	\cB_R&:=\Set{u\in \R^{d}\suchthat \frac{1}{\sqrt{n}}\Norm{Xu}= R}\,.
	\end{align*}  
	Suppose that the set size $m$ from the well-spread property satisfies $m \ge 4R^2 n$ and
	\[
	R \ge 100\cdot \sqrt{\frac{\nu k\log d+\log\Paren{2/\delta}}{\lambda \cdot \alpha^2 \cdot n}}\,.
	\]	
	Then with probability at least $1-\delta/2$, for all $u\in \cB_R \cap \cS_4(\overline{\Omega})$,
	\begin{align*}
		\bm F_2(\beta)\geq & \bm F_2(\bm y-X\beta^*)+\iprod{\nabla \bm F_2(\beta^*), u}
		+0.01\cdot \alpha n \cdot \frac{1}{n}\Norm{X u}^2\,.
	\end{align*}
	\begin{proof}
		Denote $\cC_R = \cB_R \cap \cS_4(\overline{\Omega})$ and let
		 $u\in \cC_R$. By \cref{lem:second-order-behavior-huber}, 
		\begin{align*}
			\bm F_2(\beta^*+u)-\bm F_2(\beta^*)-\iprod{\nabla \bm F_2(\beta^*),u}&\geq \frac{1}{2}
			\underset{i\in [n]}{\sum}\iprod{X_i,u}^2\ind{\Abs{\bm \eta_i}\leq 1}\cdot \ind{\Abs{\iprod{X_i, u}}\leq 1}\,.
		\end{align*}
		 Note that for any $u\in \cB_R$ there are at most $4R^2n$ coordinates of $Xu$ larger than $1/4$ in absolute value, and since $X$ is well-spread for sets of size $m = 4R^2n$,
		\begin{align*}
			\underset{i\in [n]}{\sum}\iprod{X_i,u}^2 \ind{\iprod{X_i,u}^2\leq 1/4}\geq \frac{1}{4}\Snorm{Xu}\,.
		\end{align*}
		Thus
		\begin{align*}
			E:=\E \underset{i\in [n]}{\sum}\iprod{X_i,u}^2\ind{\Abs{\bm \eta_i}\leq 1}\cdot \ind{\iprod{X_i,u}^2\leq 1/4 }\geq\frac{1}{4} \cdot  \alpha\cdot \Snorm{Xu} = \frac{\alpha R^2 n}{4}
		\end{align*}
		We now bound the deviation. We have
		\begin{align*}
			\text{for all }i \in [n]\,,\qquad &\iprod{X_i,u}^2\ind{\Abs{\bm \eta_i}\leq 1}\cdot \ind{\iprod{X_i,u}^2\leq 1/4}\leq 1\\
			\text{and}\qquad &	\E \underset{i\in [n]}{\sum}\Brac{\iprod{X_i,u}^2\cdot \ind{\iprod{X_i,u}^2\leq 1/4}
				\cdot\Paren{\ind{\Abs{\bm \eta_i}\leq 1}-\alpha_i}}^2\\
			&\leq \E \underset{i\in [n]}{\sum}\iprod{X_i,u}^2 \cdot \ind{\iprod{X_i,u}^2\leq 1/4}\cdot
			\ind{\Abs{\bm \eta_i}\leq 1}\\
			&\leq E\,.
		\end{align*}
		Applying Bernstein's inequality \cref{fact:bernstein}
		\begin{align*}
			\bbP \Paren{ \underset{i\in [n]}{\sum}\iprod{X_i,u}^2\cdot \ind{\iprod{X_i,u}^2\leq 1/4}\cdot \Paren{\ind{\Abs{\bm \eta_i}\leq 1}-\alpha_i}\geq t\cdot \sqrt{E} + t^2}\leq \exp\Set{-t^2/4}\,.
		\end{align*}
		It remains to extend uniformly this bound  over all $u\in \cC_R$.
		By \cref{lem:eps-net-l1-norm} there exists an $\Paren{\eps\cdot R}$-net $\cN_{\eps R}(\cC_R)$ of size 
		$\exp\Brac{\frac{256\nu k\log d}{\lambda\eps^2}}$ (recall that $s = 4\sqrt{k/\lambda}$). 
		Thus for any $u \in \cC_R$ there exists  $u' \in \cN_{\eps R}(\cC_R)$ such that $\frac{1}{\sqrt{n}}\Norm{X\Paren{u-u'}}\leq \eps R$ and consequently
		\begin{align*}
			\underset{i\in [n]}{\sum}\iprod{X_i,u'}^2
			\cdot\ind{\Abs{\bm \eta_i}\leq 1,\,\iprod{X_i,u'}^2\leq \tfrac 1 4}\leq& \underset{i\in [n]}{\sum}\iprod{X_i,u'}^2
			\cdot\ind{\Abs{\bm \eta_i}\leq 1,\,\iprod{X_i,u}^2\leq 1,\, \iprod{X_i,u'}^2\leq \tfrac 1 4} 
			+\eps^2R^2n
			\\
			\leq&2\underset{i\in [n]}{\sum}\iprod{X_i,u}^2
			\cdot\ind{\Abs{\bm \eta_i}\leq 1,\,\iprod{X_i,u}^2\leq 1,\,\iprod{X_i,u'}^2\leq \tfrac 1 4} 
			+\eps^2R^2n\\
			&+2\underset{i\in [n]}{\sum}\iprod{X_i,u'-u}^2
			\cdot\ind{\Abs{\bm \eta_i}\leq 1,\,\iprod{X_i,u}^2\leq 1 ,\,\iprod{X_i,u'}^2\leq \tfrac 1 4}\\
			\leq&2\underset{i\in [n]}{\sum}\iprod{X_i,u}^2
			\cdot\ind{\Abs{\bm \eta_i}\leq 1,\,\iprod{X_i,u}^2\leq 1 }
			+3\eps^2R^2n\,.
		\end{align*}
		The first inequality holds since each term at the first sum that doesn't appear in the second sum corresponds to the index $i\in[n]$ such that $\iprod{X_i,u-u'}^2 \ge 1/4$, and since each term is bounded by $1/4$, their sum is bounded by $\sum_{i\in[n]} \iprod{X_i,u-u'}^2\le \eps^2R^2n$.
		
				Setting $\eps=\sqrt{\alpha}/4$ and
				taking a union bound, with probability at least $1-\delta/2$ for all unit vectors $u \in \cL_{k,R}$ we get
		\begin{align*}
			\underset{i\in [n]}{\sum}\iprod{X_i,u}^2\ind{\Abs{\bm \eta_i}\leq 1} \ind{{\iprod{X_i, u}}^2 \leq 4}
			\geq& 
			\frac{E}{2} - \frac{3\eps^2R^2n}{2} 
			\\&\qquad-\sqrt{E}\frac{\sqrt{64\nu k\log d+  
						4\log\paren{\frac{2}{\delta} } } }  {\sqrt{\lambda}\eps} 
			\\&\qquad- 
				\frac{128\nu k\log d + 4\log\paren{\frac{2}{\delta} } } {\lambda\eps^2} 
			\\
			\ge& 0.01\cdot \alpha \cdot R^2 n\,.
		\end{align*}
	\end{proof}
\end{lemma}

\paragraph{Putting things together} We  combine the above results with \cref{thm:meta-theorem}.

\begin{proof}[\textbf{Proof of \cref{thm:oblivious-sparse-regression}}]
	By \cref{lem:subgradient-bound-regression} and \cref{lem:local-strong-convexity-sparse-regression}, we can apply \cref{thm:meta-theorem} with 
$\gamma = 100\sqrt{\nu\cdot n\Paren{\log d + \log \Paren{2/\delta}}}$, $\kappa=0.01\cdot \alpha\cdot n$ and $s = 4\sqrt{k/\lambda}$.
	It follows that for
	\begin{align*}
		R\gtrsim\sqrt{\frac{\nu\cdot k\cdot \Paren{\log d + \log(2/\delta)}}{\lambda \cdot\alpha^2\cdot  n}}\,,
	\end{align*}
	the estimator $\hat{\bm \beta}$   defined in \cref{eq:huber-loss-regression-technical}  with probability $1-\delta$ satisfies $\frac{1}{n}\Norm{X\Paren{\hat{\bm \beta}-\beta^*}}\leq  R$. Taking $\delta = d^{-10}$, we get the desired bound. 
	Since $\hat{\bm\beta}-\beta \in \cS_4({\Omega})$, we also get the desired parameter error $\Norm{\hat{\bm\beta}-\beta }\le R/\sqrt{\lambda}$.
\end{proof}

	\section{Sparse linear regression with Gaussian design (\cref{thm:oblivious-sparse-regression-gaussian-design})}\label{sec:oblivious-gaussian-design}
In this section we will prove \cref{thm:oblivious-sparse-regression-gaussian-design}. As before, we will use \cref{thm:meta-theorem}.
Recall that in this setting, our model looks as follows:
\begin{align*}
\bm y = \bm X \beta^* + \eta
\end{align*}
where $\bm X \in \R^{n \times d}$ is a \emph{random} matrix whose rows $\bm X_1, \ldots, \bm X_n$ are \iid $N(0, \Sigma)$ and $\eta \in \R^n$ is a \emph{deterministic} vector such that $\alpha n$ coordinates have absolute value bounded by 1.
We restate \cref{thm:oblivious-sparse-regression-gaussian-design} here for completeness: 

\begin{theorem}[Restatement of \cref{thm:oblivious-sparse-regression-gaussian-design}]\label{thm:oblivious-sparse-regression-gaussian-design-restatement}
	Let $\beta^* \in \R^d$ be an unknown $k$-sparse vector and let $\bm X$ be a $n$-by-$d$ random matrix with \iid rows $\bm X_1,\ldots \bm X_n \sim N(0, \Sigma)$ for a positive definite matrix $\Sigma$.
	Further, let $\eta\in \R^n$ be a deterministic vector with $\alpha\cdot n$ coordinates bounded by $1$ in absolute value.	
	Suppose that $n\gtrsim \frac{\nu(\Sigma)\cdot k \log d}{\sigma_{\min}(\Sigma) \cdot \alpha^2}$,
	where $\nu(\Sigma)$ is the maximum diagonal entry of $\Sigma$ and $\sigma_{\min}(\Sigma)$ is its smallest eigenvalue.
	Then, with probability at least $1-d^{-10}$ over $\bm X$, given $\bm X$ and $\bm y=\bm X\beta^*+\eta$, the 
	estimator \cref{eq:results-estimator-regression} satisfies
	\begin{equation*}
		\frac{1}{n} \Norm{\bm X\Paren{ \hat{\bm  \beta}- \beta^*}}^2 \leq O\Paren{\frac{\nu(\Sigma) \cdot k\log d}{\sigma_{\min}(\Sigma) \cdot \alpha^2 \cdot n}} 
		\mbox{\qquad and\qquad} 
		\Norm{{ \hat{\bm  \beta}- \beta^*}}^2 \leq O\Paren{\frac{\nu(\Sigma) \cdot k\log d}{\sigma_{\min}^2(\Sigma) \cdot \alpha^2 \cdot n}} \,.
	\end{equation*}
\end{theorem}

As in the previous section, we assume $d \ge 2$ since for $d=1$ \cref{thm:oblivious-sparse-regression-gaussian-design} is true (since the probability $1-d^{-10} = 0$ in this case).

First, we bound the gradient of Huber loss. Then, to prove restricted local strong convexity of Huber loss, we show that the values of the empirical covariance (as a quadratic form) on approximately $k$-sparse vectors are well-concentrated near the values of the actual covariance. The proof first appeared in \cite{JMLR:v11:raskutti10a}, but they only stated the result in terms of a lower bound on the values of empirical covariance and did not discuss an upper bound, though the proof of the upper bound is very similar. Then we use this concentration to prove well-spreadness and restricted local strong convexity.

Recall that $\bm F_2(\beta) = \sum_{i=1}^n f_2\Paren{\bm y_i- \iprod{\bm X_i, \beta}}$, where
\[
f_2(t):=\begin{cases}
\frac{1}{2}t^2&\text{for }\abs{t}\leq 2\,,\\
2\abs{t}-2& \text{otherwise}.
\end{cases}
\]

\paragraph{Gradient bound for Gaussian design}
\begin{lemma}\label{lem:subgradient-bound-regression-gaussian}
	Consider the settings of \cref{thm:oblivious-sparse-regression-gaussian-design}. Then with probability at least $1-\delta/2$
	\begin{align*}
	\Normm{\nabla \bm F_2(\beta^*)}\leq 4\sqrt{\nu(\Sigma) \cdot n\cdot \Paren{ \log d + \log(2/\delta)}}\,.
	\end{align*}
\end{lemma}
\begin{proof}
	By definition of the Huber loss and choice of the Huber penalty
	\begin{align*}
	\nabla \bm F_2(\beta^*)= \nabla \Paren{\sum_{i=1}^n f_2(\bm y_i - \iprod{X_i,\beta^*})}= z^\top \bm X
	\end{align*}
	where $z$ is an $n$-dimensional vector whose entries $f_2'(\eta_i)$ are bounded by $2$ in absolute value.
	Since $\frac{1}{\norm{z}}\Sigma^{-1/2}\transpose{\bm X} z = \bm g \sim N(0,1)^n$,
	\[
	\Norm{z^\top \bm X} = \Norm{z}\cdot \Norm{\Sigma^{1/2} \bm g} \le 
	2\sqrt{n} \cdot \sqrt{\nu(\Sigma)\cdot \Paren{2\log d + 4\log(2/\delta)} }\,,
	\]
	where we used the union bound over all $j\in [d]$ and the standard tail bounds for Gaussian variables 
	$\Paren{\Sigma^{1/2} \bm g}_j$ whose variance is $\Sigma_{jj}$.
\end{proof}

\paragraph{Concentration of empirical covariance on approximately $k$-sparse vectors}
To prove well-spreadness and restricted local strong convexity in case of Gaussian design $\bm X$,
we will need the fact that for all approximately $k$-sparse vectors $u$, 
$\frac{1}{n}\Norm{\bm Xu}^2 \approx \Norm{\Sigma^{1/2}u}^2$ as long as 
$n\gtrsim \frac{\nu(\Sigma) k\log d}{\sigma_{\min}(\Sigma)}$.
Formally, we will use the following theorem:
\begin{theorem}\label{thm:approx-k-sparse-concentration}
	Let $\bm X$ be a $n$-by-$d$ random matrix with \iid rows $\bm X_1,\ldots \bm X_n \sim N(0, \Sigma)$, where $\Sigma$ is a positive definite matrix. Suppose that for some $K\ge 1$, $n\ge 1000 \cdot \frac{\nu(\Sigma)}{\sigma_{\min}(\Sigma)}\cdot K\log d$. Then 
	with probability at least $1-\exp\paren{-n/100}$, for all $u\in \R^d$ such that $\normo{u}\le\sqrt{K}\norm{u}$,
	\begin{equation}
	\frac{1}{2}\,\Norm{\Sigma^{1/2}u} \le \frac{1}{\sqrt{n}}\Norm{\bm Xu} \le 
	2\,\Norm{\Sigma^{1/2}u}\,.
	\end{equation}
\end{theorem}

The first inequality of \cref{thm:approx-k-sparse-concentration} was shown in \cite{JMLR:v11:raskutti10a} (see also \cite{wainwright_2019}, section 7.3.3), and the second inequality can be proved in a very similar way. For completeness, we provide a proof of second inequality.

\begin{proof}[Proof of the second inequality of \cref{thm:approx-k-sparse-concentration}]
Since the inequality is scale invariant, it is enough to show it for $u\in \R^d$ such that $\Norm{\Sigma^{1/2} u} = 1$.
For $s>0$ denote 
\begin{align*}
\cU_s := \Set{u\in \R^d \suchthat \Norm{\Sigma^{1/2} u} = 1\,, \normo{u}\le s}
\mbox{\quad and\quad}  \cM_s\Paren{\bm X} := \sup_{u\in \cU_s} \frac{1}{\sqrt{n}}\Norm{\bm X u}\,.
\end{align*}
First, we bound the expectation of $\cM_s\Paren{\bm X}$:
\begin{lemma}
\[
\E \cM_s\Paren{\bm X} \le 1 + 2s\sqrt{\frac{\nu(\Sigma) \log d}{n}}\,.
\]
\end{lemma}
\begin{proof}
Consider Gaussian process $\bm W_{u,v} = \transpose{v} \bm X u$ for $(u,v)\in \cU_s\times S^{n-1}$, where $S^{n-1}$ is a unit sphere in $\R^n$. Denote $\cP = \cU_s\times S^{n-1}$.
Our goal is to bound $\frac{1}{\sqrt{n}}\E \sup\limits_{\paren{u,v}\in\cP} \bm W_{u,v}$.

Denote $\bm G = \bm X \Sigma^{-1/2}$. For all $(u,v), (\tilde{u}, \tilde{v})\in \cP$,
\begin{align*}
\E\Paren{\bm W_{u,v}-\bm W_{\tilde{u},\tilde{v}}}^2 
&= \E\iprod{\transpose{\bm X}, u\transpose{v} -\tilde{u}\transpose{\tilde{v}}}^2
\\&= \E\iprod{\transpose{\bm G}, \Sigma^{1/2}u\transpose{v} -\Sigma^{1/2}\tilde{u}\transpose{\tilde{v}}}^2
\\&=\Normf{\Sigma^{1/2}u\transpose{v} -\Sigma^{1/2}\tilde{u}\transpose{\tilde{v}}}^2\,.
\end{align*}

Now consider another Gaussian process $\bm Z_{u,v} = \transpose{\bm g} \Sigma^{1/2}u + \transpose{\bm h} v$, where $\bm g\sim N(0, \Id_d)$ and $\bm h\sim N(0,\Id_n)$.  For all $(v,u), (\tilde{v}, \tilde{u})\in\cP$,
\begin{align*}
\E\Paren{\bm Z_{u,v}-\bm Z_{\tilde{u},\tilde{v}}}^2 
&=  \E \iprod{\bm g, \Sigma^{1/2}\Paren{u-\tilde{u}}}^2 + \E \iprod{\bm h, {v-\tilde{v}}}^2
\\&= \Norm{\Sigma^{1/2}u-\Sigma^{1/2}\tilde{u}}^2 + \Norm{v-\tilde{v}}^2\,.
\end{align*}

Note that for every quadruple of unit vectors $x, \tilde{x}\in \R^d\,,y,\tilde{y}\in \R^n$,
\begin{align*}
\normf{x\transpose{y} - \tilde{x}\transpose{\tilde{y}}}^2 
=& 
\normf{\paren{x-\tilde{x}}\transpose{y} + \tilde{x}\paren{\transpose{y}- \transpose{\tilde{y}}}}^2 
\\=&
 \snorm{y}\snorm{x-\tilde{x}}+\snorm{\tilde{x}}\snorm{y-\tilde{y}} + 2\Tr{y\transpose{\paren{x-\tilde{x}}} \tilde{x}\paren{\transpose{y}- \transpose{\tilde{y}}}}
\\=&
\snorm{x-\tilde{x}} + \snorm{y-\tilde{y}} + 2\Paren{\iprod{x,\tilde{x}}-\snorm{\tilde{x}}}\cdot\Paren{\snorm{y}-\iprod{y,\tilde{y}}}
\\\le&
\snorm{x-\tilde{x}} + \snorm{y-\tilde{y}}\,.
\end{align*}

Hence for all $(u,v), (\tilde{u}, \tilde{v})\in \cP$, 
$\E\Paren{\bm W_{u,v}-\bm W_{\tilde{u},\tilde{v}}}^2 \le \E\Paren{\bm Z_{u,v}-\bm Z_{\tilde{u},\tilde{v}}}^2$, and by Sudakov--Fernique theorem \cref{fact:sudakov-fernique},
\[
\E \sup\limits_{\paren{u,v}\in\cP}  \bm W_{u,v} \le \E \sup\limits_{\paren{u,v}\in\cP}  \bm Z_{u,v}\,.
\]

Therefore, it is enough to bound $\E\sup\limits_{u\in \cU_s}\transpose{\bm g} \Sigma^{1/2}u + \E\sup\limits_{\norm{v}=1} \transpose{\bm h} v$. The second term is just an expectation of $\chi$ distributed variable, and can be bounded using Jensen's inequality:
\[
\E\sup\limits_{\norm{v}=1} \transpose{\bm h} v = \E \norm{\bm h} \le \sqrt{\E\norm{\bm h}^2} \le \sqrt{n}\,.
\]

The first term can be bounded as follows:
\[
\E\sup\limits_{u\in \cU_s}\transpose{\bm g} \Sigma^{1/2}u \le \E\normo{u}\cdot\norm{\Sigma^{1/2}\bm g}_{\max}
\le s\E\norm{\Sigma^{1/2}\bm g}_{\max} \le 2s\sqrt{\nu(\Sigma)\log d}\,,
\]
where we used \cref{fact:max-norm-subgaussian-entries} to bound the max norm of a vector $\Sigma^{1/2}\bm g$ whose entries are $\nu(\Sigma)$-subgaussian.
Dividing by $\sqrt{n}$, we get the desired bound.
\end{proof}

 Now, we bound the deviation of $\cM_s\Paren{\bm X}$:
 \begin{lemma} For all $t\ge 0$,
 	\[
 	\Pr\Brac{\Abs{\cM_s\Paren{\bm X} - \E\cM_s\Paren{\bm X}}\ge t}\le 2\exp\Brac{-nt^2/2}\,.
 	\]
 \end{lemma}
\begin{proof}
For $A\in\R^{n\times d}$ denote $\cF_s(A) = \sqrt{n}\cdot \cM_s( A\Sigma^{1/2}) = \sup_{u\in \cU_s} \Norm{A\Sigma^{1/2}u}$. 
Note that for all $A, B\in \R^{n\times d}$,
\begin{align*}
\cF_s(A)  - \cF_s(B) 
&\le \sup_{u\in \cU_s}   \Paren{\Norm{A\Sigma^{1/2}u} -  \Norm{B\Sigma^{1/2}u}}
\\&\le \sup_{u\in \cU_s}   \Norm{\Paren{A-B}\Sigma^{1/2}u}
\\&\le \Norm{A-B} \le \Normf{A-B}\,.
\end{align*}
Hence $\cF_s$ is  $1$-Lipschitz, and by \cref{fact:lipschitz-gaussian}, for all $\tau\ge 0$,
\[
\Pr\Brac{\Abs{\cF_s\Paren{\bm G} - \E\cF_s\Paren{\bm G}}\ge \tau}\le 2\exp\Brac{-\tau^2/2}\,,
\]
where $\bm G = \bm X \Sigma^{-1/2}$ is a matrix with \iid standard Gaussian entries. Taking $\tau = t\sqrt{n}$, we get the desired bound.
\end{proof}

Taking $t=0.2$, we conclude that with probability at least $1-2\exp\paren{-0.02n}$, 
\[
\cM_s\Paren{\bm X} \le 1.2 +  2s\sqrt{\frac{\nu(\Sigma) \log d}{n}}\,.
\]
For $s = \sqrt{K/\sigma_{\min}(\Sigma)}$ this bound implies that for all $u$ such that $\Norm{\Sigma^{1/2} u} = 1$ and $\normo{u}\le\sqrt{K}\norm{u}$, with probability at least $1-2\exp\paren{-0.02n}$,
\[
\frac{1}{\sqrt{n}}\Norm{\bm Xu} \le 1.2 + 2\sqrt{\frac{\nu(\Sigma) \,K\log d}{\sigma_{\min}(\Sigma) \cdot n}}\le 1.3\,,
\]
and we get the desired bound.
\end{proof}

\paragraph{Well-spreadness of Gaussian matrices}
If $n \gtrsim \frac{\nu(\Sigma)}{\sigma_{\min}(\Sigma)}\cdot k\log d$, then an $n\times d$ random matrix $\bm X$ with \iid rows $\bm X_1,\ldots, \bm X_n\sim N(0,\Sigma)$  satisfies the RE-property with parameter $\sigma_{\min}(\Sigma)/ 4$ over all sets of size $k$ where $\sigma_{\min}(\Sigma)$ is the smallest eigenvalue of $\Sigma$ (it is a consequence of \cref{thm:approx-k-sparse-concentration}).
Also, norms of columns of $\bm X$ are bounded by $O\Paren{\sqrt{\nu(\Sigma) n}}$ with high probability.
Hence, $\bm X$ satisfies Assumption 1 and 2 of \cref{thm:oblivious-sparse-regression-restatement}, with high probability.

In the next lemma, we show that it also satisfies the last assumption, namely the well-spreadness assumption:
\begin{lemma}\label{lem:well-spreadness-gaussian}
	Let $\bm X$ be a $n$-by-$d$ random matrix with \iid rows $\bm X_1,\ldots \bm X_n \sim N(0, \Sigma)$, where $\Sigma$ is a positive definite matrix. Suppose that for some $K\ge 1$, $n\ge 10^6 \cdot \frac{\nu(\Sigma)}{\sigma_{\min}(\Sigma)}\cdot K\log d$. Then 
with probability at least $1-\exp\paren{-n/1000}$, for all $u\in \R^d$ such that $\normo{u}\le\sqrt{K}\norm{u}$ and for all sets $S\subseteq [n]$ of size $\lceil 0.999n \rceil$, 
\begin{equation}
\Norm{\bm X_Su} \ge\frac{1}{2}\Norm{\bm Xu} \,.
\end{equation}
\end{lemma}
\begin{proof}
	For a set $M\subseteq[n]$ of size at most $n/1000$ independent of $\bm X$, \cref{thm:approx-k-sparse-concentration} implies that $\Norm{\bm X_Mu} \le  0.1\sqrt{n}\cdot \Norm{\Sigma^{1/2}u}$ and $\Norm{\bm Xu} \ge  0.5\sqrt{n}\cdot \Norm{\Sigma^{1/2}u}$ with probability at least $1-2\exp\paren{-n/100}$. Using a union bound over all sets $M$ of size $n -\lceil 0.999n \rceil$, with probability 
	\[
	 1-2\exp\Brac{-n/100 + n\log(1000e)/1000} \ge 1 - \exp\paren{-n/1000}\,,
	\]
	we get 
	\[
	\Norm{\bm X_Mu}^2 \le 0.1 \Norm{\bm Xu}^2\,.
	\]
	Since for $S = [n]\setminus M$, $\Norm{\bm Xu}^2 = \Norm{\bm X_Mu}^2 + \Norm{\bm X_Su}^2$, we get the desired bound.
\end{proof}

Now we can prove restricted strong convexity.

\paragraph{Restricted local strong convexity of Huber loss for Gaussian design}

\begin{lemma}\label{lem:local-strong-convexity-sparse-regression-gaussian}
	\newcommand{\placeholder}{2^i}
	Consider the settings of \cref{thm:oblivious-sparse-regression-gaussian-design}. Let $0<\delta<1, R> 0$. Define 
	\begin{align*}
	\bm \cB_R&:=\Set{u\in \R^{d}\suchthat \frac{1}{\sqrt{n}}\Norm{\bm X\Paren{u}}= R}\,.
	\end{align*}
		Suppose that $R\le \frac{1}{200}$.
		Then with probability at least $1-3\exp\paren{-\alpha n / 1000}$, for all $u\in \bm\cB_R \cap \cS_4(\Omega)$,
	\begin{align*}
	\bm F_2(\beta^*+u)\geq \bm F_2(\beta^*)+\iprod{\nabla\bm  F_2(\beta^*), u}
	+\frac{\alpha n}{200} \cdot \frac{1}{n}\Norm{\bm X u}^2\,.
	\end{align*}
\end{lemma}
\begin{proof}
	Let
	$u\in \cB_R \cap \cS_4(\Omega)$, where $\Omega$ is the support of $\beta^*$. By \cref{lem:second-order-behavior-huber}, 
    \[	
	\bm F_2(\beta^*+u)-\bm F_2(\beta^*)-\iprod{\nabla \bm F_2(\beta^*),u}
	\geq 
	\frac{1}{2}
	\underset{i\in [n]}{\sum}\iprod{\bm X_i,u}^2\ind{\Abs{\eta_i}\leq 1}\cdot \ind{\Abs{\iprod{\bm X_i, u}}\leq 1}\,.
	\]
	Denote $A=\Set{i\in [n] \suchthat \abs{\eta_i} \le 1}$. Matrix $\bm X_A$ is an $\alpha n \times d$ random matrix with \iid rows $\bm X_j \sim N(0,\Sigma)$. By \cref{thm:approx-k-sparse-concentration}, with probability $1-2\exp\paren{-\alpha n / 100}$, 
	\begin{align*}
	16\alpha R^2n &= 16\alpha \Norm{\bm X u}^2 
	\\&\ge 4\alpha n \Norm{\Sigma^{1/2} u}^2 
	\\&\ge  \Norm{\bm X_Au}^2 \ge \frac{\alpha n}{4}\Norm{\Sigma^{1/2} u}^2 
	\\&\ge \frac{\alpha}{16} \Norm{\bm X u}^2 
	\\&= \frac{\alpha}{16}R^2n\,.
	\end{align*}
	
	By \cref{lem:well-spreadness-gaussian}, with probability $1-\exp\paren{-\alpha n / 1000}$, $\bm X_A$ satisfies well-spread property for sets of size $\alpha n /1000 $ and for all $u\in \cS_4(\cK)$. Since number of entries of $\bm X_A u$ which are larger than $1$ is at most $16 \alpha R^2n \le \alpha n /1000$, we get
	\begin{align*}
	\underset{i\in [n]}{\sum}\iprod{\bm X_i,u}^2\ind{\Abs{ \eta_i}\leq 1}\cdot \ind{\Abs{\iprod{\bm X_i, u}}\leq 1} 
	&= \underset{i\in A}{\sum}\iprod{\bm X_i,u}^2 \ind{\Abs{\iprod{\bm X_i, u}}\leq 1} 
	\\&\ge\frac{1}{4}\Norm{\bm X_Au}^2 \ge \frac{\alpha}{64} R^2n\,.
	\end{align*}
	
 Hence with probability at least $1-3\exp\paren{-\alpha n / 1000}$ we get the desired bound.
\end{proof}

\paragraph{Putting everything together}
Let's check that the conditions of \cref{thm:meta-theorem} are satisfied for $\Omega = \overline{\Omega} = \supp(\beta^*)$ and $\cE(u) = \frac{1}{\sqrt{n}}\Norm{\bm Xu}$.
Decomposability is obvious.
As a consequence of \cref{thm:approx-k-sparse-concentration} $\bm X$ satisfies the RE-property with $\lambda \ge \sigma_{\min}(\Sigma)/4$ with probability at least $1-\exp\paren{-n/100}$, so contraction is satisfied with $s = 8\sqrt{k/\sigma_{\min}(\Sigma)}$. By \cref{lem:subgradient-bound-regression-gaussian}, gradient is bounded by $15\sqrt{\nu(\Sigma) \cdot n\cdot \Paren{ \log d}}$ with probability $1-d^{-10}/2$. By \cref{lem:local-strong-convexity-sparse-regression-gaussian}, with probability at least $1-3\exp(n/1000)$, Huber loss satisfies restricted local strong convexity with parameter $\kappa = 0.01\alpha n$. Hence for
\[
n\gtrsim \frac{\nu(\Sigma) \cdot k \log d}{\sigma_{\min}(\Sigma) \cdot \alpha^2} 
\mbox{\qquad and\qquad} 
R \gtrsim \sqrt{\frac{\nu(\Sigma) \cdot k \log d}{\sigma_{\min}(\Sigma) \cdot \alpha^2\cdot n}}
\]
and since then we have $\hat{\bm \beta} - \beta^* \in \cB_R \cap \cS_4(\Omega)$ the estimator $\hat{\bm \beta}$   defined in \cref{eq:huber-loss-regression-technical} satisfies $\frac{1}{\sqrt{n}} \Norm{\bm X\Paren{ \hat{\bm  \beta}- \beta^*}} \le R$ with probability at least $1-d^{-10}$. 	Since $\hat{\bm\beta}-\beta \in \cS_4({\Omega})$, we also get the desired parameter error $\Norm{\hat{\bm\beta}-\beta }\le 2R/\sqrt{\sigma_{\min}(\Sigma)}$.
	\section{Optimal fraction of inliers for principal component analysis under oblivious noise (\cref{thm:IT_lower_bound})}
\label{sec:lower-bound-pca}

In this section we prove \cref{thm:IT_lower_bound}. 
Recall that a successful $(\epsilon,\delta)$-weak recovery algorithm (where $\eps,\delta\in(0,1)$) for PCA  is an algorithm that takes $\bm Y$ as input and returns a matrix $\hat{\bm L}$ such that $\Normf{\hat{\bm L}-L^*}\leq \epsilon\cdot \rho$ with probability at least $1-\delta$ (where $\rho$, $\bm Y$ and $L^*$ are as in \cref{thm:oblivious-pca}). 

Let's restate \cref{thm:IT_lower_bound}:

\begin{theorem}[Restatement of \cref{thm:IT_lower_bound}]\label{thm:IT_lower_bound-restatement}
	Let $\bm Y = L^* + \bm N \in \R^{n\times n}$, where $\rank\paren{L^*} = r$, $\normm{L^*}\le \rho/n$ and the entries of $\bm N$ are independent and symmetric about zero. Let $\zeta \ge 0$.
	
	Then
	there exists a universal constant $C_0>0$ such that for every $0<\epsilon<1$ and $0<\delta<1$, if $\alpha:=\min_{i,j\in[n]}\mathbb{P}[|\bm N_{i,j}|\leq\zeta]$ satisfies $\alpha< C_0\cdot (1-\epsilon^2)^2\cdot(1-\delta)\cdot\sqrt{r/n}$,
	and $n$ is large enough, then it is information-theoretically impossible to have a successful $(\epsilon,\delta)$-weak recovery algorithm. The problem remains information-theoretically impossible (for the same regime of parameters) even if we assume that $L^*$ is incoherent; more precisely, even if we know that $L^*$ has incoherence parameters that are as good as those of a random flat matrix of rank $r$, the theorem still holds.
\end{theorem}

More in detail, we construct distributions over $L^*$ and $\bm N$ such that the assumptions of the theorem are satisfied and if $\alpha< C_0\cdot (1-\epsilon^2)^2\cdot(1-\delta)\cdot\sqrt{r/n}$, weak recovery is not possible.

We will assume without loss of generality that $0\le \zeta \le \rho/n = 1$. Indeed, weak recovery property is scale invariant, so we can assume $\rho = n$. We can assume $\zeta \le1$ since if the theorem is true for $\zeta = 1$, then it is true for all $\zeta > 1$.

\subsection*{A generative model for the hidden matrix}

In the following, we will denote the all-zeros vector of dimension $n$ as $\bm 0_n$. Similarly, we will denote the all-ones vector of dimension $n$ as $\bm 1_n$.

For the sake of simplicity, we will assume that $\frac{n}{r}$ is an integer.\footnote{All the subsequent proofs can be adapted for a general $r$ with minor modifications.} We will divide the the matrix $\bm L^*$ into $r$ blocks of $\frac{n}{r}\times n$ sub-matrices.

For every $1\leq k\leq r$, let $u_k$ be an arbitrary but fixed and deterministic vector in the set $\left\{\bm 0_{(k-1)\cdot \frac{n}{r}}\right\}\times \{-1,+1\}^{\frac{n}{r}}\times \left\{\bm 0_{(r-k)\cdot \frac{n}{r}}\right\}$, and let $\mathbf{v}_{k}$ be a random flat vector chosen uniformly from $\{-1,+1\}^n$. We further assume that the random vectors $\{\mathbf{v}_{k}\}_{1\leq k\leq r}$ are mutually independent. The hidden matrix $\bm L^*$ is constructed as follows:\footnote{For the general case in which $\frac{n}{r}$ may not be an integer, we divide $\bm{L}^*$ into $r$ blocks of disjoint sub-matrices of dimensions $\lfloor\frac{n}{r}\rfloor\times n$ and $\lceil\frac{n}{r}\rceil\times n$.}

$$\bm L^* = \sum_{k=1}^{r}u_k\cdot \mathbf{v}_{k}^T.$$

Note that $\bm L^*$ is a flat matrix, i.e., $\bm L^*\in\{-1,+1\}^{n\times n}$. Furthermore, the rank of $\bm L^*$ is at most $r$, and with high probability, $\bm L^*$ is incoherent with parameter $\mu \le O\Paren{\log n}$.

\subsection*{The noise distribution}

Let $(\bm N_{ij})_{i,j\in[n]}$ be i.i.d. random variables that are sampled according to the distribution
\begin{equation}\label{eq:pca-noise-distribution}
\mathbb{P}[\bm N_{i,j}=\ell]=\begin{cases}\displaystyle\frac{\xi\sqrt{r}}{2\sqrt{n}-\xi\sqrt{r}}\left(1-\xi\sqrt{\frac{r}{n}}\right)^{|\ell|/2}\quad&\text{if $\ell$ is even}\\\\0\quad&\text{otherwise},\end{cases}
\end{equation}
where $0<\xi\le 1/2$ is a constant. Furthermore, we assume that $\bm N$ is independent from $\bm L^*$. The distribution of $N$ is symmetric and satisfies
$$\alpha:=\mathbb{P}[|\bm N_{ij}|\leq 1]=\mathbb{P}[\bm N_{ij}=0]=\frac{\xi\sqrt{r}}{2\sqrt{n}-\xi\sqrt{r}}=\Theta\left(\xi\sqrt{\frac{r}{n}}\right).$$

Define
$$\bm Y=\bm L^*+\bm N.$$

\subsection*{Upper bound on the mutual information}

\begin{lemma}
\label{lem:Upper_Bound_Mutual_Info}
The mutual information $I(\bm L^*;\bm Y)$ between $\bm L^*$ and $\bm Y$ can be upper bounded as follows:
\begin{align*}
I(\bm L^*;\bm Y)&\leq O\Paren{ \xi \cdot n \cdot r }\,.
\end{align*}
\begin{proof}
Notice that for every $\ell\in 2\mathbb{Z}$, we have
\begin{equation}
\label{eq:eq_prob_N_Nplus2}
\mathbb{P}[\bm N_{ij}=\ell+2]=\mathbb{P}[\bm N_{ij}=\ell]\cdot\left(1-\xi\sqrt{\frac{r}{n}}\right)^{\sign(\ell+1)},
\end{equation}

where
$$\sign(x)=\begin{cases}1\quad&\text{if }x> 0,\\-1\quad&\text{if }x<0.\end{cases}$$

For every $L^*\in\{-1,+1\}^{n\times n}$ and every $Y\in(2\mathbb{Z}+1)^{n\times n}$, we have

\begin{align*}
\mathbb{P}[\bm Y=Y|\bm L^*=L^*]&=\mathbb{P}[\bm N=Y-L^*]\\
&=\prod_{i,j}\mathbb{P}[\bm N_{ij}=Y_{ij}-L^*_{ij}]\\
&=\prod_{i,j}\mathbb{P}[\bm N_{ij}=Y_{ij}-1+1-L^*_{ij}]\\
&\stackrel{(\ast)}{=}\prod_{i,j}\left[\mathbb{P}[\bm N_{ij}=Y_{ij}-1]\cdot\left(1-\xi\sqrt{\frac{r}{n}}\right)^{\frac{1}{2}\cdot(1-L^*_{ij})\cdot \sign(Y_{ij}-1+1)}\right]\\
&=\mathbb{P}\left[\bm N=Y-\mathbf{1}_n\mathbf{1}_n^T\right]\cdot \prod_{i,j}\left(1-\xi\sqrt{\frac{r}{n}}\right)^{\frac{1}{2}\cdot(1-L^*_{ij})\cdot \sign(Y_{ij})},
\end{align*}

where $(\ast)$ follows from \cref{eq:eq_prob_N_Nplus2}. Therefore, we can write
\begin{equation}
\label{eq:eq_prob_Y_given_L}
\mathbb{P}[\bm Y=Y|\bm L^*=L^*]=\mathbb{P}\left[\bm N=Y-\mathbf{1}_n\mathbf{1}_n^T\right]\cdot f(L^*,Y),
\end{equation}
where
\begin{equation}
\label{eq:eq_prob_Y_given_L_def_f}
\begin{aligned}
f(L^*,Y)
&=\prod_{i,j}\left(1-\xi\sqrt{\frac{r}{n}}\right)^{\frac{1}{2}\cdot(1-L^*_{i,j})\cdot \sign(Y_{ij})}\\
&=\left(1-\xi\sqrt{\frac{r}{n}}\right)^{\resizebox{0.18\textwidth}{!}{$\displaystyle\frac{1}{2}\cdot\sum_{i,j}(1-L^*_{i,j})\cdot \sign(Y_{ij})$}}\\
&=\left(1-\xi\sqrt{\frac{r}{n}}\right)^{\frac{1}{2}\cdot\langle \mathbf{1}_n\mathbf{1}_n^T-L^*,\sign(Y)\rangle},
\end{aligned}
\end{equation}
where $\sign(Y)$ is the $n\times n$ matrix defined as $\sign(Y)_{i,j}=\sign(Y_{i,j})$. Furthermore, from \cref{eq:eq_prob_Y_given_L} we can deduce that for every $Y\in(2\mathbb{Z}+1)^{n\times n}$, we have
\begin{equation}
\label{eq:eq_prob_Y_unconditional}
\mathbb{P}[\bm Y=Y]=\mathbb{P}\left[\bm N=Y-\mathbf{1}_n\mathbf{1}_n^T\right]\cdot g(Y),
\end{equation}
where
\begin{equation}
\label{eq:eq_prob_Y_unconditional_def_g}
g(Y)=\mathbb{E}_{\bm L^*}[f(\bm L^*,Y)].
\end{equation}

Now from H\"{o}lder's inequality, we have
\begin{align*}
\left|\left\langle \mathbf{1}_n\mathbf{1}_n^T-L^*,\sign(Y)\right\rangle\right|
&\leq \Normn{\mathbf{1}_n\mathbf{1}_n^T-L^*}\cdot\Norm{\sign(Y)}\\
&\leq \Paren{\Normn{\mathbf{1}_n\mathbf{1}_n^T}+\Normn{L^*}} \cdot\Norm{\sign(Y)}\\
&= \Paren{n+\Normn{L^*}} \cdot\Norm{\sign(Y)}.
\end{align*}

If $L^*$ is in the support of the distribution of $\bm L^*$, then there exist $r$ vectors $\{v_k\}_{1\leq k\leq r}$ such that $v_k\in\{-1,+1\}^n$ for every $1\leq k\leq r$, and
$$L^* = \sum_{k=1}^{r}u_{k}\cdot {v}_{k}^T.$$
Therefore,
\begin{equation}
\label{eq:eq_L_normnuc_inequality}
\begin{aligned}
\Normn{L^*} &\leq \sum_{k=1}^{r}\Normn{{u}_{k}\cdot {v}_{k}^T}= \sum_{k=1}^{r}\sqrt{\frac{n}{r}}\cdot \sqrt{n}= r\frac{n}{\sqrt{r}}= n\sqrt{r}.
\end{aligned}
\end{equation}

Hence,
\begin{align*}
\left|\left\langle \mathbf{1}_n\mathbf{1}_n^T-L^*,\sign(Y)\right\rangle\right|
&\leq n\cdot(\sqrt{r}+1) \cdot\Norm{\sign(Y)}\leq 2n\sqrt{r}\cdot\Norm{\sign(Y)}.
\end{align*}

By combining this with \cref{eq:eq_prob_Y_given_L_def_f}, we get

\begin{equation}
\label{eq:eq_f_inequality}
\begin{aligned}
\left(1-\xi\sqrt{\frac{r}{n}}\right)^{n\sqrt{r}\cdot\Norm{\sign(Y)}}\leq f(L^*,Y)\leq  \left(1-\xi\sqrt{\frac{r}{n}}\right)^{-n\sqrt{r}\cdot\Norm{\sign(Y)}}.
\end{aligned}
\end{equation}

Furthermore, from \cref{eq:eq_prob_Y_unconditional_def_g} and \cref{eq:eq_f_inequality}, we get
\begin{equation}
\label{eq:eq_g_inequality}
\begin{aligned}
\left(1-\xi\sqrt{\frac{r}{n}}\right)^{n\sqrt{r}\cdot\Norm{\sign(Y)}}\leq g(Y)\leq  \left(1-\xi\sqrt{\frac{r}{n}}\right)^{-n\sqrt{r}\cdot\Norm{\sign(Y)}}.
\end{aligned}
\end{equation}

The mutual information between $\bm L^*$ and $\bm Y$ is given by
\begin{align*}
I(\bm L^*;\bm Y)&=\sum_{L^*,Y}\mathbb{P}[\bm L^*=L^*,\bm Y=Y]\cdot \log_2\frac{\mathbb{P}[\bm Y=Y|\bm L^*=L^*]}{\mathbb{P}[\bm Y=Y]}\\
&\stackrel{(\dagger)}{=}\sum_{L^*,Y}\mathbb{P}[\bm L^*=L^*,\bm Y=Y]\cdot \log_2\frac{\mathbb{P}[\bm N=Y-\mathbf{1}_n\mathbf{1}_n^T]\cdot f(L^*,Y)}{\mathbb{P}[\bm N=Y-\mathbf{1}_n\mathbf{1}_n^T]\cdot g(Y)}\\
&=\sum_{L^*,Y}\mathbb{P}[\bm L^*=L^*,\bm Y=Y]\cdot \log_2\frac{f(L^*,Y)}{g(Y)}\\
&=\mathbb{E}\left[\log_2\frac{f(\bm L^*,\bm Y)}{g(\bm Y)}\right],
\end{align*}
where $(\dagger)$ follows from \cref{eq:eq_prob_Y_given_L} and \cref{eq:eq_prob_Y_unconditional}. Now from \cref{eq:eq_f_inequality} and \cref{eq:eq_g_inequality}, we get
\begin{equation}
\label{eq:eq_mutual_information_eq1}
\begin{aligned}
I(\bm L^*;\bm Y)&\leq\mathbb{E}\left[\log_2\left(\left(1-\xi\sqrt{\frac{r}{n}}\right)^{-2n\sqrt{r}\cdot\Norm{\sign(\bm Y)}}\right)\right]\\
&=-2n\sqrt{r}\cdot\log_2\left(1-\xi\sqrt{\frac{r}{n}}\right) \cdot\mathbb{E}\left[\Norm{\sign(\bm Y)}\right]\\
&=-\frac{2}{\log 2}n\sqrt{r}\cdot\log\left(1-\xi\sqrt{\frac{r}{n}}\right) \cdot\mathbb{E}\left[\Norm{\sign(\bm Y)}\right]\\
&\stackrel{(\ddagger)}{\leq} \frac{4}{\log 2}n\sqrt{r}\cdot \xi\sqrt{\frac{r}{n}} \cdot\mathbb{E}\left[\Norm{\sign(\bm Y)}\right]\\
&= \frac{4\xi\cdot \sqrt{n}\cdot r}{\log 2}\cdot\mathbb{E}\left[\Norm{\sign(\bm Y)}\right],
\end{aligned}
\end{equation}
where $(\ddagger)$ follows from the fact that $-\log(1-t)\leq 2t$ for every $t\in [0,1/2]$.

Now let $\bm{S}=\sign(\bm Y)$. In order to estimate $\mathbb{E}\left[\Norm{\sign(\bm Y)}\right]=\mathbb{E}\left[\Norm{\bm S}\right]$, we first condition on $\bm L^* = L^*$ for a fixed $L^*$:
\begin{align*}
\mathbb{E}\left[\Norm{\bm S}\big|\bm L^* = L^*\right]&=\mathbb{E}\left[\Norm{\bm S-\mathbb{E}\left[\bm S|\bm L^* = L^*\right]+\big.\mathbb{E}\left[\bm S|\bm L^* = L^*\right]}\Big|\bm L^* = L^*\right]\\
&\leq\mathbb{E}\left[\Norm{\bm S-\mathbb{E}\left[\bm S|\bm L^* = L^*\right]\big.}\Big|\bm L^* = L^*\right]+\Norm{\mathbb{E}\left[\bm S|\bm L^* = L^*\right]}.
\end{align*}

Notice that
\begin{align*}
\mathbb{E}\left[\bm S\big|\bm L^* = L^*\right]&=\mathbb{E}\left[ \sign(\bm L^* + \bm N)\big|\bm L^* = L^*\right]\\
&=\mathbb{E}\left[ \sign(L^* + \bm N)\right]\\
&=\alpha \cdot L^*,
\end{align*}
where $$\alpha=\mathbb{P}[\bm N_{ij}=0]=\frac{\xi\sqrt{r}}{2\sqrt{n}-\xi\sqrt{r}}.$$

Therefore,
\begin{align*}
\mathbb{E}\left[\Norm{\bm S}\big|\bm L^* = L^*\right]&\leq\mathbb{E}\left[\Norm{\bm \hat{\bm S}\big.}\Big|\bm L^* = L^*\right]+\alpha\cdot \Norm{L^*},
\end{align*}
where
$$\hat{\bm S}=\bm S-\mathbb{E}\left[\bm S|\bm L^* = L^*\right]=\bm S-\alpha \cdot L^*.$$

Now given $\bm L^* = L^*$, the entries of $\hat{\bm S}$ are centered and conditionally mutually independent. Furthermore, $\Normm{\hat{\bm S}}\leq \Normm{\bm S}+\alpha\cdot\Normm{L^*}=1+\alpha\leq 2$. Therefore, by \cref{fact:spectral-norm-bounded-entries}, there is a universal constant $C\geq 2$ such that
$$\mathbb{E}\left[\Norm{\bm \hat{\bm S}\big.}\Big|\bm L^* = L^*\right]\leq C\sqrt{n}.$$

We conclude that
\begin{equation}
\label{eq:eq_ineq_Expectation_sign_Y}
\begin{aligned}
\mathbb{E}\left[\Norm{\sign(\bm Y)}\right]=\mathbb{E}\left[\Norm{\bm S}\right]\leq C\sqrt{n}+\alpha\cdot \mathbb{E}\left[\Norm{\bm L^*}\right].
\end{aligned}
\end{equation}

Now notice that $\Norm{\bm L^*}=U\cdot\mathbf{V}^T$, where $U=[u_1\;\ldots\; u_r]$ is the $n\times r$ matrix whose columns are $u_1,\ldots,u_r$, and $\mathbf{V}=[\mathbf{v}_1\;\ldots\; \mathbf{v}_r]$ is the $n\times r$ matrix whose columns are $\mathbf{v}_1,\ldots,\mathbf{v}_r$. We have:
\begin{align*}
\mathbb{E}\left[\Norm{\bm L^*}\right]&=\mathbb{E}\left[\Norm{U\cdot\mathbf{V}^T}\right]\leq \mathbb{E}\left[\Norm{U}\cdot\Norm{\mathbf{V}^T}\right] = \Norm{U}\cdot\mathbb{E}\left[\Norm{\mathbf{V}^T}\right]\\
&=\sqrt{\frac{n}{r}}\cdot \mathbb{E}\left[\Norm{\mathbf{V}^T}\right]\stackrel{(\wr)}{\leq} \sqrt{\frac{n}{r}}\cdot C\sqrt{n}= C\frac{n}{\sqrt{r}},
\end{align*}
where $(\wr)$ follows from the fact that $\mathbf{V}$ is an $n\times r$ matrix with i.i.d. zero-mean entries and satisfying $\Normm{\mathbf{V}}= 1$ and \cref{fact:spectral-norm-bounded-entries}. By inserting this in \cref{eq:eq_ineq_Expectation_sign_Y}, we get
\begin{align*}
\mathbb{E}\left[\Norm{\sign(\bm Y)}\right]&\leq C\sqrt{n}+\frac{\xi\sqrt{r}}{2\sqrt{n}-\xi\sqrt{r}}\cdot C \frac{n}{\sqrt{r}}\\
&\leq C\sqrt{n}+\frac{\sqrt{r}}{\sqrt{n}}\cdot C \frac{n}{\sqrt{r}}\\
&= 2C \sqrt{n},
\end{align*}

By combining this with \cref{eq:eq_mutual_information_eq1}, we get

\begin{align*}
I(\bm L^*;\bm Y)&\leq \frac{4\xi\cdot \sqrt{n}\cdot r}{\log 2} \cdot 2C \sqrt{n}\\
&\leq \frac{8C\xi}{\log 2}\cdot n\cdot r\\
&=O(\xi\cdot n\cdot r).
\end{align*}
\end{proof}
\end{lemma}

\subsection*{Successful weak-recovery reduces entropy}
\begin{lemma}
\label{lem:lem_Fano_ineq_application}
If there exists a $(\delta,\epsilon)$-successful weak recovery algorithm that takes $\bm Y=\bm L^*+\bm N$ as input and returns a matrix $\hat{\bm L}$ as output in such a way that
$$\Pr\left[\Normf{\hat{\bm L}-\bm L^*}\leq \epsilon\cdot n\right]\geq 1-\delta,$$
then the mutual information between $\bm L^*$ and $\bm Y$ can be lower bounded as follows:
\begin{align*}
I( \bm L^* ; \bm Y)\geq \frac{(1-\epsilon^2)^2}{8\log 2}\cdot(1-\delta)\cdot n\cdot r - 1.
\end{align*}
\begin{proof}
Define the set
$$\Omega=\left\{\sum_{k=1}^r u_k v_k^T:\;\forall k\in[r],\;v_k\in\mathbb{R}^n\;\text{and}\;\Normm{v_k}\leq 1\right\}.$$

It is easy to see that $\Omega$ is a closed and convex set. Let $\hat{\bm L}_{\Omega}$ be the orthogonal projection of $\hat{\bm L}$ onto $\Omega$. Since $\bm L^*\in \Omega$, we have $\Normf{\hat{\bm L}_\Omega-\bm L^*}\leq \Normf{\hat{\bm L}-\bm L^*}$. Therefore,

$$\Pr\left[\Normf{\hat{\bm L}_\Omega-\bm L^*}\leq \epsilon\cdot n\right]\geq \Pr\left[\Normf{\hat{\bm L}-\bm L^*}\leq \epsilon\cdot n\right]\geq 1-\delta.$$

Using an inequality that is similar to the standard Fano-inequality, we will show that the existence of a successful weak-recovery algorithm implies a linear decrease in the entropy of the random vectors $(\mathbf{v}_k)_{k\in[r]}$.

Define the random variable $\bm Z$ as follows:
$$\bm Z=\ind{\Normf{\hat{\bm L}_\Omega-\bm L^*}\leq \epsilon\cdot n}.$$

Furthermore, for every $L\in\Omega$, define
$$B_{L,\epsilon}=\left\{(v_k)_{k\in[r]}\in\{-1,+1\}^{n\cdot r}:\;\Normf{L-\sum_{k=1}^r u_k\transpose{v_k}}\leq\epsilon\cdot n\right\}.$$
Clearly, if $\bm Z=1$, then $(\mathbf{v}_k)_{k\in [r]}\in B_{\hat{\bm L}_\Omega,\epsilon}$.

Let $H\left((\mathbf{v}_k)_{k\in[r]}\big|\bm\hat{\bm L}_\Omega\right)$ be the conditional entropy of $(\mathbf{v}_k)_{k\in[r]}$ given $\bm\hat{\bm L}_\Omega$. We have:
\begin{align*}
H\left((\mathbf{v}_k)_{k\in[r]}\middle|\bm\hat{\bm L}_\Omega\right)&\leq H\left(\bm Z,(\mathbf{v}_k)_{k\in[r]}\middle|\bm\hat{\bm L}_\Omega\right)\\
&= H\left(\bm Z\middle|\bm\hat{\bm L}_\Omega\right) + H\left((\mathbf{v}_k)_{k\in[r]}\middle|\bm\hat{\bm L}_\Omega,\bm Z\right)\\
&\leq \resizebox{0.78\textwidth}{!}{$\displaystyle H\left(\bm Z\right) + H\left((\mathbf{v}_k)_{k\in[r]}\middle|\bm\hat{\bm L}_\Omega,\bm Z=0\right)\cdot\Pr[\bm Z=0]+H\left((\mathbf{v}_k)_{k\in[r]}\middle|\bm\hat{\bm L}_\Omega,\bm Z=1\right)\cdot\Pr[\bm Z=1]$}\\
&\leq 1 + n\cdot r\cdot\Pr[\bm Z=0]+(1-\Pr[\bm Z=0])\cdot H\left((\mathbf{v}_k)_{k\in[r]}\big|\bm\hat{\bm L}_\Omega,\bm Z=1\right),
\end{align*}
where the last inequality follows from the fact that $\bm Z$ is a binary random variable (and hence $H(\bm Z)\leq\log_2(2)=1$), and the fact that $(\mathbf{v}_k)_{k\in[r]}\in \{-1,+1\}^{n\cdot r}$, which implies that $H\left((\mathbf{v}_k)_{k\in[r]}\big|\bm\hat{\bm L}_\Omega,\bm Z=0\right)\leq \log_2\big|\{-1,+1\}^{n\cdot r}\big|=n\cdot r$.

Since $\Pr[\bm Z=0]\leq \delta$ and $H\left((\mathbf{v}_k)_{k\in[r]}\big|\bm\hat{\bm L}_\Omega,\bm Z=1\right)\leq \log_2\big|\{-1,+1\}^{n\cdot r}\big|=n\cdot r$, we have
\begin{align*}
H\left((\mathbf{v}_k)_{k\in[r]}\middle|\bm\hat{\bm L}_\Omega\right)&\leq 1 + n\cdot r+(1-\delta)\cdot H\left((\mathbf{v}_k)_{k\in[r]}\big|\bm\hat{\bm L}_\Omega,\bm Z=1\right).
\end{align*}

Now notice that
\begin{align*}
H\left((\mathbf{v}_k)_{k\in[r]}\middle|\bm\hat{\bm L}_\Omega,\bm Z=1\right)
\stackrel{(\ast)}{\leq}  \log_2\left|B_{\hat{\bm L}_\Omega, \eps}\right|
\leq \max_{L\in\Omega} \; \log_2\left|B_{L,\epsilon}\right|,
\end{align*}
where $(\ast)$ follows from the fact that given $\bm Z=1$, we have $(\mathbf{v}_k)_{k\in [r]}\in B_{\hat{\bm L}_\Omega,\epsilon}$. On the other hand, for every $L\in\Omega$, we have
\[
\log_2\left|B_{L,\epsilon}\right| = n\cdot r+ \log_2\frac{\left|B_{L,\epsilon}\right|}{2^{n\cdot r}}
= n\cdot r+ \log_2 \Pr\left[(\mathbf{v}_k)_{k\in[r]}\in B_{L,\epsilon}\right],
\]
where the last equality follows from the fact that $(\mathbf{v}_k)_{k\in[r]}$ is uniformly distributed in $\{-1,+1\}^{n\cdot r}$. Therefore,

\begin{equation}
\label{eq:eq_Fano_ineq}
H\left((\mathbf{v}_k)_{k\in[r]}\middle|\bm\hat{\bm L}_\Omega\right)\leq  1 + n\cdot r+(1-\delta)\cdot \max_{L\in\Omega}\;\log_2 \Pr\left[(\mathbf{v}_k)_{k\in[r]}\in B_{L,\epsilon}\right].
\end{equation}

Now fix $L\in\Omega$ and let $(v_k)_{k\in [r]}$ be $k$ vectors in $\mathbb{R}^n$ such that $\Normm{v_k}\leq 1$ and $\displaystyle L=\sum_{k=1}^r u_k v_k^T$. We have $(\mathbf{v}_k)_{k\in[r]}\in B_{L,\epsilon}$ if and only if $\Normf{\bm L^*-L}\leq\epsilon\cdot n$. Notice that
\begin{align*}
\Normf{\bm L^*-L}^2&=\left\langle \sum_{k=1}^r u_k\cdot(\mathbf{v}_k-v_k)^T,\sum_{k'=1}^r u_{k'}\cdot(\mathbf{v}_{k'}-v_{k'})^T\right\rangle\\
&=\Tr\left(\left(\sum_{k=1}^r u_k\cdot(\mathbf{v}_k-v_k)^T\right)^T\cdot \left(\sum_{k'=1}^r u_{k'}\cdot(\mathbf{v}_{k'}-v_{k'})^T\right)\right)\\
&=\sum_{k=1}^r \sum_{k'=1}^r \Tr\left((\mathbf{v}_k-v_k)\cdot u_k^T\cdot u_{k'}\cdot(\mathbf{v}_{k'}-v_{k'})^T\right)\\
&\stackrel{(\dagger)}{=}\frac{n}{r}\cdot\sum_{k=1}^r \Tr\left((\mathbf{v}_k-v_k)\cdot(\mathbf{v}_k-v_k)^T\right)=\frac{n}{r}\cdot\sum_{k=1}^r \Norm{\mathbf{v}_k-v_k}^2\\
&=\frac{n}{r}\cdot\sum_{k=1}^r \sum_{i=1}^n (\mathbf{v}_{k,i}-v_{k,i})^2=\frac{n}{r}\cdot\sum_{k=1}^r \sum_{i=1}^n \left(\mathbf{v}_{k,i}^2+v_{k,i}^2-2v_{k,i}\cdot \mathbf{v}_{k,i}\right)\\
&\stackrel{(\ddagger)}{\geq}\frac{n}{r}\cdot\left(n\cdot r - 2\sum_{k=1}^r \sum_{i=1}^n v_{k,i}\cdot \mathbf{v}_{k,i}\right)\,,
\end{align*}
where $(\dagger)$ follows from the fact that $(u_k)_{k\in[r]}$ are orthogonal to each other, and $\Norm{u_k}^2=\frac{n}{r}$ for every $k\in [r]$. Note that $(\mathbf{v}_{k,i})_{1\leq i\leq n}$ and $(v_{k,i})_{1\leq i\leq n}$ are the entries of $\mathbf{v}_k$ and $v_k$, respectively. $(\ddagger)$ follows from the fact that $\mathbf{v}_k\in\{-1,+1\}^n$ for every $k\in [r]$. Therefore,

$$\Normf{\bm L^*-L}^2\geq n^2 - \frac{2n}{r}\cdot \sum_{k=1}^r \sum_{i=1}^n v_{k,i}\cdot \mathbf{v}_{k,i}\,,$$
which implies that
\begin{align*}
\Pr\left[(\mathbf{v}_k)_{k\in[r]}\in B_{L,\epsilon}\right]&=\Pr\left[\Normf{\bm L^*-L}^2\leq\epsilon^2\cdot n^2\right]\\
&\leq \Pr\left[n^2 - \frac{2n}{r}\cdot \sum_{k=1}^r \sum_{i=1}^n v_{k,i}\cdot \mathbf{v}_{k,i}\leq\epsilon^2\cdot n^2\right]\\
&= \Pr\left[\sum_{k=1}^r \sum_{i=1}^n v_{k,i}\cdot \mathbf{v}_{k,i}\geq \frac{1}{2}\cdot (1-\epsilon^2)\cdot n\cdot r\right].
\end{align*}

Note that the random variables $(v_{k,i}\cdot \mathbf{v}_{k,i})_{k\in[r],i\in[n]}$ are independent. Moreover, for every $1\leq k\leq r$ and every $1\leq i\leq n$, we have $$\E[v_{k,i}\cdot \mathbf{v}_{k,i}]=0.$$ Furthermore, since $\Normm{v_k}\leq 1$ and $\mathbf{v}_k\in\{-1,+1\}^n$, the random variables $(v_{k,i}\cdot \mathbf{v}_{k,i})_{k\in[r],i\in[n]}$ can be uniformly bounded as
$$|v_{k,i}\cdot \mathbf{v}_{k,i}|\leq |v_{k,i}|\leq 1.$$

It follows from Hoeffding's inequality \cref{fact:hoeffding} that
\begin{align*}
\Pr\left[(\mathbf{v}_k)_{k\in[r]}\in B_{L,\epsilon}\right]&\leq \exp\left(-\frac{(1-\epsilon^2)^2\cdot n^2\cdot r^2}{8\cdot n\cdot r}\right)\\
&= \exp\left(-\frac{(1-\epsilon^2)^2}{8}\cdot n\cdot r\right).
\end{align*}

Since this is true for every $L\in\Omega$, we get from \cref{eq:eq_Fano_ineq} that
\begin{equation*}
H\left((\mathbf{v}_k)_{k\in[r]}\middle|\bm\hat{\bm L}_\Omega\right)\leq  1 + n\cdot r-(1-\delta)\cdot\frac{(1-\epsilon^2)^2}{8\log 2}\cdot n\cdot r.
\end{equation*}

Therefore, the mutual information between $(\mathbf{v}_k)_{k\in[r]}$ and $\bm\hat{\bm L}_\Omega$ satisfies:
\begin{align*}
I\left((\mathbf{v}_k)_{k\in[r]};\bm\hat{\bm L}_\Omega\right)&=H\left((\mathbf{v}_k)_{k\in[r]}\right)-H\left((\mathbf{v}_k)_{k\in[r]}\big|\bm\hat{\bm L}_\Omega\right)\\
&\geq n\cdot r -  1 - n\cdot r+(1-\delta)\cdot\frac{(1-\epsilon^2)^2}{8\log 2}\cdot n\cdot r\\
&= \frac{(1-\epsilon^2)^2}{8\log 2}\cdot(1-\delta)\cdot n\cdot r - 1.
\end{align*}

Now since $(\mathbf{v}_k)_{k\in[r]} \to \bm L^* \to \bm Y \to \hat{\bm  L} \to \hat{\bm  L}_\Omega$ is a Markov chain, it follows from the data-processing inequality that
\begin{align*}
I( \bm L^* ; \bm Y) =I\left((\mathbf{v}_k)_{k\in[r]};\bm Y\right) \geq I\left((\mathbf{v}_k)_{k\in[r]};\bm\hat{\bm L}_\Omega\right)\geq \frac{(1-\epsilon^2)^2}{8\log 2}\cdot(1-\delta)\cdot n\cdot r - 1.
\end{align*}
\end{proof}
\end{lemma}

\subsection*{Putting everything together}

Now we are ready to prove \cref{thm:IT_lower_bound}:
\begin{proof}[Proof of \cref{thm:IT_lower_bound}]
From \cref{lem:Upper_Bound_Mutual_Info} and \cref{lem:lem_Fano_ineq_application} we can deduce that if there exists a $(\delta,\epsilon)$-successful weak recovery algorithm then we must have
\begin{align*}
\frac{8C\xi}{\log 2}\cdot n\cdot r\geq I( \bm L^* ; \bm Y)\geq \frac{(1-\epsilon^2)^2}{8\log 2}\cdot(1-\delta)\cdot n\cdot r - 1.
\end{align*}
Therefore, if $n$ is large enough and
\begin{align*}
\xi< \frac{(1-\epsilon^2)^2}{64 C}\cdot(1-\delta) - \frac{\log 2}{8C} \cdot \frac{1}{r \cdot n},
\end{align*}
it is impossible to have a $(\delta,\epsilon)$-successful weak recovery algorithm. Now since $\displaystyle\alpha=\Theta\left(\xi\sqrt{\frac{r}{n}}\right)$, we get the result.
\end{proof}

	\addcontentsline{toc}{section}{Bibliography}
	\bibliographystyle{amsalpha}
	\bibliography{bib/mathreview,bib/dblp,bib/scholar,bib/custom}
	
	\clearpage
	\appendix
	\section{Facts about Huber loss}
\label{sec:additional-tools}

\begin{fact}[Integration by parts for absolutely continuous functions]\label{lem:integration-by-parts}
	Let $F,G:\R\to \R$ be absolutely continuous functions, i.e. there exist locally integrable functions $f,g:\R\to \R$ such that for all $a,b\in \R$,
	\[
	\int_a^b f(t)\,\mathrm  dt =  F(b)-F(a) \qquad \mbox{and} \qquad \int_a^b g(t)\,\mathrm  dt =  G(b) -G(a)\,.
	\]
	Then for all $a,b\in \R$,
	\[
	\int_a^b f(t)G(t) \,\mathrm  dt = F(b)G(b)-F(a)G(a) - \int_a^b F(t)g(t) \,\mathrm  dt \,.
	\]
\end{fact}
\begin{proof}
	\begin{align*}
	\int_a^b f(t)G(t) \,\mathrm  dt 
	&= G(a)\cdot\Paren{F(b)-F(a)} + \int_a^b f(t) \int_a^b \ind{\tau\in [a,t]} g(\tau) \,\mathrm  d\tau \,\mathrm  dt \\
	&= G(a)\cdot\Paren{F(b)-F(a)} + \int_a^b g(\tau)  \int_a^b f(t)  \ind{t\in [\tau,b]} \,\mathrm dt \,\mathrm d\tau \\
	&= G(a)\cdot\Paren{F(b)-F(a)} + \int_a^b g(\tau)  \cdot\Paren{F(b)-F(\tau)} \,\mathrm d\tau \\
	&= G(a)\cdot\Paren{F(b)-F(a)} + F(b)\Paren{G(b)-G(a)} - \int_a^b g(\tau)F(\tau)   \,\mathrm d\tau \\
	&= F(b)G(b)-F(a)G(a) - \int_a^b F(t)g(t) \,\mathrm  dt \,.
	\end{align*}
\end{proof}

\begin{lemma}[Second order behavior of Huber-loss function]\label{lem:second-order-behavior-huber}
	Let $h > 0$.
	For all $\eta, \delta\in \R$, and all $0\le \tau\le h$,
	\begin{align*}
		f_h(\eta+\delta)-f_h(\eta)-f'_h(\eta)\cdot \delta\geq \frac{\delta^2}{2}\ind{\abs{\eta}\leq h-\tau}\cdot \ind{\Abs{\delta}\leq \tau}\,.
	\end{align*}
	\begin{proof}
			Consider $g:\R \to \R$ defined as $g(t) = f_h\Paren{\eta + t\cdot \delta}$.  
			Note that for all $a,b\in \R$,
			\[
			f'_h(\eta + b\delta) - f'_h(\eta + a\delta) = \int_{\eta + a\delta}^{\eta + b\delta} \ind{\abs{x}\le h} \,\mathrm dx\,.
			\]
			Changing the variable $x = \eta + t\delta$, we get
			\[
			g'(b)-g'(a) = \delta^2\int_a^b \ind{\abs{\eta + t\delta}\le h} \,\mathrm dt\,.
			\]
			By \cref{lem:integration-by-parts},
			\[
			\delta^2\int_{0}^1 \ind{\abs{\eta + t\delta}\le h} \cdot (1-t)  \, \mathrm dt = -g'(0) + g(1)-g(0)\,.
			\]
			Note that $g(0) = f_h(\eta)$, $g(1) = f_h(\eta+\delta)$ and $g'(0) = \delta f'_h(\eta)$. 
			Since for all $0 \le \tau\le h$, $\ind{\abs{\eta + t\delta}\le h} \ge \ind{\abs{\eta}\leq h-\tau}\cdot \ind{\Abs{\delta}\leq \tau}$ and $\int_0^1 (1-t)\, \mathrm dt = 1/2$, we get the desired bound.
	\end{proof}
\end{lemma}

\section{Tools for Probabilistic Analysis}
\label{sec:random-matrices-bounds}
This section contains some technical results needed for the proofs in the main body of the paper.

\begin{fact}[Chernoff's inequality, \cite{vershynin_2018}]\label{fact:chernoff}
	Let $\bm \zeta_1,\ldots, \bm \zeta_n$ 
	be independent Bernoulli random variables such that 
	$\Pr\Paren{\bm \zeta_i = 1} = \Pr\Paren{\bm\zeta_i = 0} = p$. 
	Then for every $\Delta > 0$,
	\[
	\Pr\Paren{\sum_{i=1}^n \bm\zeta_i \ge pn\Paren{1+ \Delta} } 
	\le 
	\Paren{ \frac{e^{-\Delta} }{ \Paren{1+\Delta}^{1+\Delta} } }^{pn}\,.
	\]
	and for every $\Delta \in (0,1)$,
	\[
	\Pr\Paren{\sum_{i=1}^n \bm\zeta_i \le pn\Paren{1- \Delta} } 
	\le 
	\Paren{ \frac{e^{-\Delta} }{ \Paren{1-\Delta}^{1-\Delta} } }^{pn}\,.
	\]
\end{fact}

\begin{fact}[Hoeffding's inequality, \cite{wainwright_2019}]\label{fact:hoeffding}
	Let $\bm z_1,\ldots, \bm z_n$
	be mutually independent random variables such that for each $i\in[n]$,
	$\bm z_i$ is supported on $\brac{-c_i, c_i}$ for some $c_i \ge 0$. 
	Then for all $t\ge 0$,
	\[
	\Pr\Paren{\sum_{i=1}^n \Paren{\bm z_i - \E \bm z_i} \ge t} 
	\le \exp\Paren{-\frac{t^2}{2\sum_{i=1}^n c_i^2}}\,,
	\]
	and
	\[
	\Pr\Paren{\Abs{\sum_{i=1}^n \Paren{\bm z_i - \E \bm z_i}} \ge t} 
	\le 2\exp\Paren{-\frac{t^2}{2\sum_{i=1}^n c_i^2}}\,.
	\]
\end{fact}
\begin{fact}[Bernstein's inequality,
	\cite{wainwright_2019}]\label{fact:bernstein}
	Let $\bm z_1,\ldots, \bm z_n$
	be mutually independent random variables such that for each $i\in[n]$,
	$\bm z_i$ is supported on $\brac{-B, B}$ for some $B\ge 0$. 
	Then for all $t\ge 0$,
	\[
	\Pr\Paren{{\sum_{i=1}^n \Paren{\bm z_i - \E \bm z_i}} \ge t} 
	\le \exp\Paren{-\frac{t^2}{2\sum_{i=1}^n \E \bm z_i^2 + \frac{2Bt}{3}}}\,.
	\]
\end{fact}

\begin{fact}[Subgaussian maxima, \cite{wainwright_2019}]\label{fact:max-norm-subgaussian-entries}
	Let $d\ge 2$ be an integer and let $\bm z$ be a $d$-dimensional random vector with zero mean $\sigma$-subgaussian entires. Then
	\[
	\E \norm{\bm z}_{\max} \le2\sigma\sqrt{\log d}\,. 
	\]
\end{fact}

\begin{fact}[Lipschitz functions of Gaussian vectors, \cite{wainwright_2019}]\label{fact:lipschitz-gaussian}
Let $\bm g \sim N(0,1)^m$ for some $m\in \N$ and let $F:\R^m\to \R$ be $L$-Lipschitz with respect to Euclidean norm, where $L>0$. Then for all $t\ge 0$,
\[
\Pr\Brac{\Abs{F(\bm g) - \E F(\bm g)} \ge t} \le 2\exp\Paren{-\frac{t^2}{2L^2}}\,.
\]
\end{fact}

\begin{fact}[Spectral norm of Gaussian matrices, \cite{wainwright_2019}]\label{fact:spectral-norm-gaussian}
Let $\bm W\sim N(0,1)^{n\times d}$. Then
\[
\E \Norm{\bm W} \le \sqrt{n} + \sqrt{d}\,.
\]
Moreover, for all $t\ge 0$,
\[
\Pr\Brac{\Norm{\bm W} \ge \sqrt{n} + \sqrt{d} + t} \le 2\exp\paren{-t^2/2}\,.
\]
\end{fact}

\begin{fact}[Spectral norm of matrices with bounded independent zero-mean entries, \cite{random_matrices_rudelson_vershynin}]\label{fact:spectral-norm-bounded-entries}
	Let $\bm M$ be an $n$-by-$n$ random matrix with independent zero-mean entries  $\bm M_{ij}$ supported on $[-1,1]$. Then
	\[
	\E \Norm{\bm M} \le \Paren{2+o(1)}\sqrt{n}
	\]
	as $n\to \infty$.
	Moreover, for all $t\ge 0$,
	\[
	\Pr\Brac{\Norm{\bm M} \ge \E \Norm{\bm M} + \sqrt{2\pi} + t} \le 2\exp\paren{-t^2/2}\,.
	\]
\end{fact}

\begin{fact}[Sudakov--Fernique theorem, \cite{gaussian_processes_adler}]\label{fact:sudakov-fernique}
Let $\Theta$ be a compact subset of $\R^m$, where $m\in \N$. 
Let $\bm W_\theta$ and $\bm Z_\theta$ be real-valued sample-continuous zero-mean Gaussian processes indexed by elements of $\Theta$. 
Suppose that $\forall \theta,\theta'\in \Theta$,
$\E \Paren{\bm W_\theta - \bm W_{\theta'}}^2\le \E \Paren{\bm Z_\theta - \bm Z_{\theta'}}^2$. Then
\[
\E\sup_{\theta\in \Theta}\bm W_\theta \le \E\sup_{\theta\in \Theta}\bm Z_\theta\,.
\]
\end{fact}

\begin{fact}[Sudakov Minoration, \cite{wainwright_2019}] \label{fact:sudakov-minoration}
	Let $\Set{\bm g_\theta\suchthat \theta\in \Theta}$ be a zero-mean Gaussian process indexed by elements of some non-empty set $\Theta$. Let $\rho:\Theta\times \Theta\rightarrow [0,\infty)$ be a (pseudo)metric
	$\rho(\theta,\theta'):=\Paren{\E\paren{\bm g_\theta-\bm g_{\theta'}}^2}^{1/2}$. Then
	\begin{align*}
		\E \sup_{\theta \in \cT} \bm g_\theta\geq \sup_{\eps > 0}\frac{\eps}{2}\sqrt{\log \Abs{\cN_{\eps, \rho}( \Theta)}} \,,
	\end{align*}
	where $\Abs{\cN_{\eps, \rho}( \Theta)}$ is the minimal size of $\eps$-net in $\Theta$ with respect to $\rho$.
\end{fact}

\end{document}